\documentclass[acmsmall]{acmart}
\usepackage{algorithm}
\usepackage[noend]{algpseudocode}
\makeatletter
\usepackage{amsmath}
\usepackage{svg}
\def\algbackskip{\hskip-\ALG@thistlm}
\makeatother
\algnewcommand\algorithmicforeach{\textbf{for each:}}
\algnewcommand\ForEach{\item[ \algorithmicforeach]}
\algdef{S}[FOR]{ForEach}[1]{\algorithmicforeach\ #1\ \algorithmicdo}
\usepackage{caption}
\usepackage{subcaption}
\usepackage{array}
\usepackage{multirow}

\newcommand\MyBox[2]{
  \fbox{\lower0.75cm
    \vbox to 1.7cm{\vfil
      \hbox to 1.7cm{\hfil\parbox{1.4cm}{#1\\#2}\hfil}
      \vfil}%
  }%
}
\AtBeginDocument{%
  \providecommand\BibTeX{{%
    \normalfont B\kern-0.5em{\scshape i\kern-0.25em b}\kern-0.8em\TeX}}}

\setcopyright{acmlicensed}
\acmJournal{TIST}
\acmYear{2021} \acmVolume{1} \acmNumber{1} \acmArticle{1} \acmMonth{1} \acmPrice{15.00}\acmDOI{10.1145/3467977}

\begin{document}

\title{Physics-Guided Abnormal Trajectory Gap Detection}


\author{Arun Sharma}
\email{sharm485@umn.edu}
\orcid{0002-6908-6960}
\affiliation{
  \institution{University of Minnesota, Twin Cities}
  \city{Minneapolis}
  \state{Minnesota}
  \country{USA}
}

\author{Shashi Shekhar}
\email{shekhar@umn.edu}
\affiliation{%
  \institution{University of Minnesota, Twin Cities}
  \city{Minneapolis}
  \state{Minnesota}
  \country{USA}
}

\renewcommand{\shortauthors}{Arun Sharma and Shashi Shekhar}

\begin{abstract}
Given trajectories with gaps (i.e., missing data), we investigate algorithms to identify abnormal gaps in trajectories which occur when a given moving object did not report its location, but other moving objects in the same geographic region periodically did. The problem is important due to its societal applications, such as improving maritime safety and regulatory enforcement for global security concerns such as illegal fishing, illegal oil transfers, and trans-shipments. The problem is challenging due to the difficulty of bounding the possible locations of the moving object during a trajectory gap, and the very high computational cost of detecting gaps in such a large volume of location data. The current literature on anomalous trajectory detection assumes linear interpolation within gaps, which may not be able to detect abnormal gaps since objects within a given region may have traveled away from their shortest path. In preliminary work, we introduced an abnormal gap measure that uses a classical space-time prism model to bound an object's possible movement during the trajectory gap and provided a scalable memoized gap detection algorithm (Memo-AGD). In this paper, we propose a Space Time-Aware Gap Detection (STAGD) approach to leverage space-time indexing and merging of trajectory gaps. We also incorporate a Dynamic Region Merge-based (DRM) approach to efficiently compute gap abnormality scores. We provide theoretical proofs that both algorithms are correct and complete and also provide analysis of asymptotic time complexity. Experimental results on synthetic and real-world maritime trajectory data show that the proposed approach substantially improves computation time over the baseline technique.
\end{abstract}



\maketitle

\section{Introduction}
Given multiple trajectory gaps and a signal coverage map based on historical object activity, we find possible abnormal gaps in activity where moving objects (e.g., ships) may have behaved abnormally, such as not reporting their locations in an area where other ships historically did report their location. The top half of Figure \ref{fig:problemstatement} shows the problem's \textit{input}, which includes a map of the signal coverage area (grey cells) for a set of derived historical trajectories (Figure  \ref{subfig:Input1}) and trajectory gaps $G_{1}$, ..., and $G_{10}$ (Figure \ref{subfig:Input2}). Cells are created by mapping the spatial extent of latitude and longitude coordinates across all historical location traces. A cell is assigned to a signal coverage (i.e., grey) if it contains historic location traces whose total number exceeds a certain threshold $\theta$ (more details in section \ref{subsection:basicconcept}). The bottom half of the figure shows the output where gaps $G_{1}$, $G_{4}$ $G_{6}$ and $G_{9}$ are entirely outside the signal coverage map (Figure \ref{subfig:Intermediate Output}), indicating weak signal coverage. In contrast, the rest of the gaps overlap the signal coverage map. The absence of location reporting in an area known to have signal coverage may be interpreted as intentional behavior by a ship that temporarily switched off its location broadcasting device. The output shown in Figure \ref{subfig:Intermediate Output} reflects on \textit{intermediate output} stage where trajectory gaps ($G_{1}$, ..., $G_{10}$) are modeled in the form of geo-ellipses \cite{pfoser1999capturing,miller1991modelling} along with their intersections with the signal coverage map. Figure \ref{subfig:Intermediate Output} also shows two pairs of gaps $G_{2}$, $G_{3}$ and $G_{7}$, $G_{8}$, that have highly intersecting regions (i.e., dark grey cells), which suggests two ships in a rendezvous potentially engaging in illegal activity. The gaps entirely outside the signal coverage area have been filtered out. Figure \ref{subfig:Final Output} shows the \textit{final output} where gap pairs $G_{2}$, $G_{3}$ and $G_{7}$, $G_{8}$ are merged to their overlapping regions and gap $G_{10}$ is filtered out since it did not meet the user-defined \textit{priority threshold}.

\begin{figure}[htp]
    \centering 
\begin{subfigure}{0.35\textwidth}
  \includegraphics[width=\linewidth]{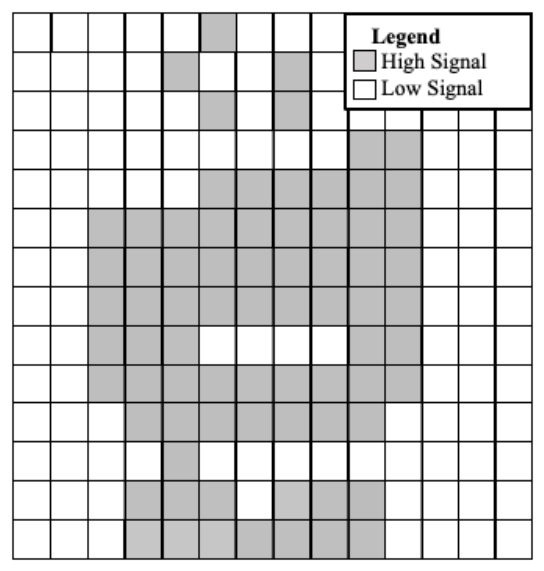}
  \captionsetup{justification=centering}
  \caption{\textbf{Input 1:} Signal coverage map (SCM)\\}
  \label{subfig:Input1}
\end{subfigure}\hfil 
\begin{subfigure}{0.35\textwidth}
  \includegraphics[width=\linewidth]{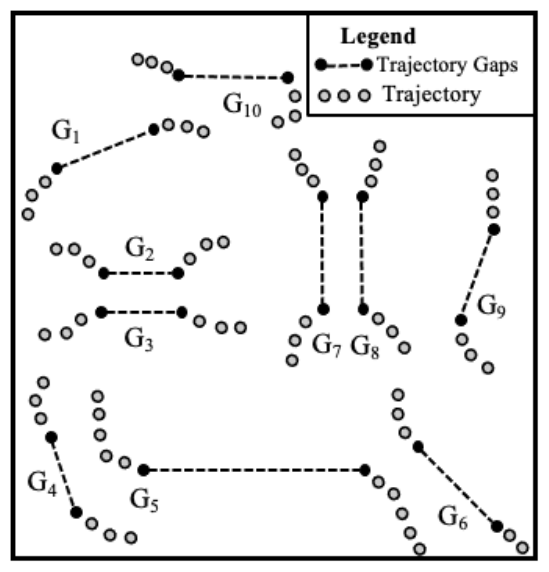}
  \captionsetup{justification=centering}
  \caption{\textbf{Input 2:} Trajectories with trajectory gaps ($G_{i}$)\\}
  \label{subfig:Input2}
\end{subfigure}\hfil 
\medskip
\begin{subfigure}{0.35\textwidth}
  \includegraphics[width=\linewidth]{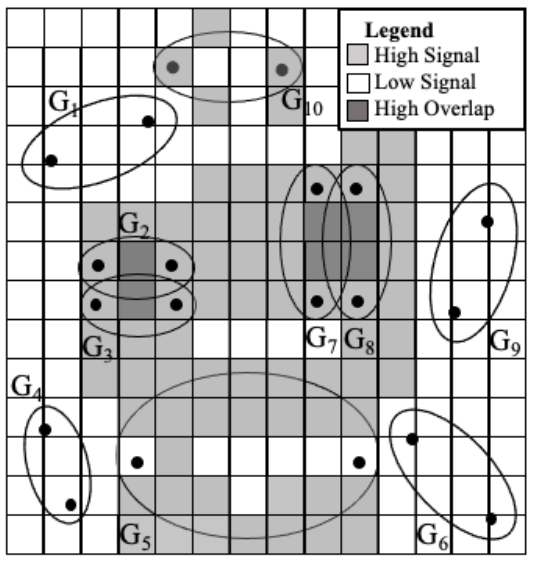}
  \captionsetup{justification=centering}
  \caption{\textbf{Intermediate Output}: Abnormal gap computation\\}
  \label{subfig:Intermediate Output}
\end{subfigure}\hfil 
\begin{subfigure}{0.35\textwidth}
  \includegraphics[width=\linewidth]{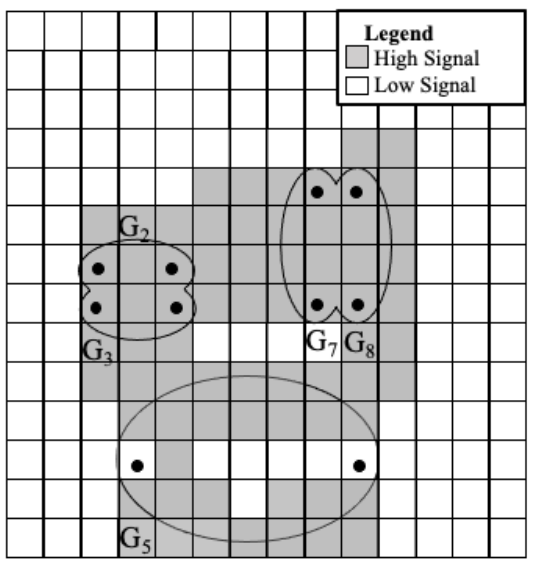}
  \captionsetup{justification=centering}
  \caption{\textbf{Final Output:} Summarized abnormal gaps}
  \label{subfig:Final Output}
\end{subfigure}\hfil 
\caption{An illustration of the abnormal gap region problem (Best in color).}
\label{fig:problemstatement}
\end{figure}





Societal applications for analyzing trajectory gaps are related to maritime safety, homeland security, epidemiology, and public safety. For instance, illegal fishing, trans-shipments, illegal oil transfers, etc., are important for global maritime safety and regulatory enforcement. Such activities can be restricted and managed by identifying frequent missing signals from GPS trajectories of oil or fishing vessels. Table \ref{ApplicationDomain} lists many use cases from various domains. For instance, in astrophysics, estimating lost planets by tracking their orbital trajectories has been traditionally studied by Gauss \cite{astronomy}. Signal strength while tracking marine animals in the deep sea tends to get lost, and a more accurate estimation model is needed. Similarly, analyzing trajectory gaps can help to find lost areal vehicles by local search space on where the object could have landed or crashed \cite{spaceshuttle}. Such trajectory reconstruction techniques require accurate physics-based estimation in order to reduce time-intensive post-processing because it often requires manual inspection of gaps in trajectories spread over a large geographic space.


\begin{table}\scriptsize
\footnotesize
\centering
\caption{Application Domain and Use Cases Examples for Abnormal Gap Detection}
\label{ApplicationDomain}
\begin{tabular}{p{3.0cm}p{9cm}p{5cm}}
\hline
Application Domain  &  Example Use Cases\\ 
\hline
Astrophysics   & Manual tracking of astronomical objects (e.g., comets) which are no longer visible or temporarily lost in space.\\ 
\hline
Marine Biology  & Tracking marine animals (e.g., whales, white sharks) trajectory movements in deep seas where signals tend to get lost for several minutes.\\ 
\hline
Aviation and Defense & Localizing the search space to track lost areal vehicles in open space.\\ 
\hline
\end{tabular}
\end{table}

There are two challenges to this problem. First, many interpolation or inference methods can lead to missed patterns since moving objects do not always travel on a straight path. In addition, probabilistic methods such as Gaussian processes provide some estimation of where the object could have potentially deviated but do not holistically capture the object's movement in the trajectory gap. Our approach, which is based on a space-time prism, identifies a larger spatial region surrounding a gap and captures all object's possible movement capabilities, which Gaussian processes or linear interpolation-based methods could have missed. Computing such deviations captures missing patterns in real-world scenarios (e.g., illegal fishing near marine protected habitats). Second, there is the challenge of handling large data volumes distributed over a considerable geographical space. For example, MarineCadastre \cite{aisdataUS}, an open-source automatic identification system (AIS) dataset, contains records with more than 30 attributes (e.g., speed, draft) for thousands of ships taken every minute over a decade, resulting in terabytes of data.

\begin{figure}[ht]
    \centering
    \includegraphics[width=0.50\textwidth]{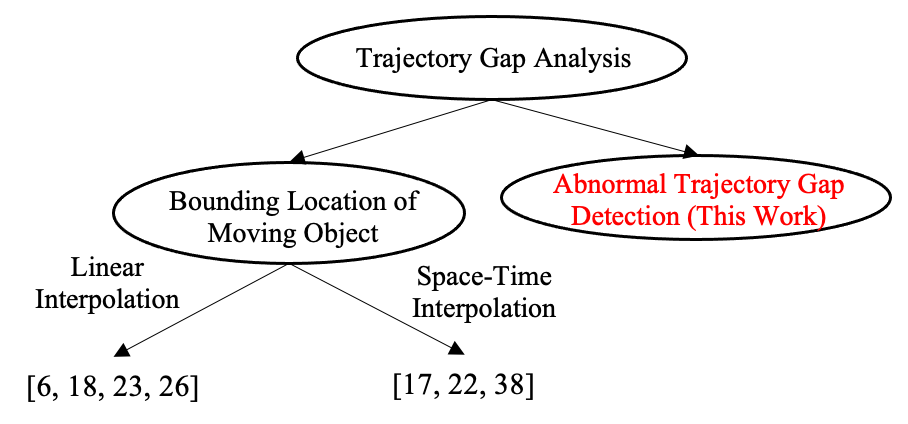}
    \caption{Comparison of Related Work\\}
    \label{fig:ComparisonofRelatedWork}
\end{figure}


Most of the studies \cite{pallotta2013vessel,chen2013iboat,lei2016framework} on trajectory gaps that bound the location of the moving object is based on shortest path discovery. For instance, \cite{lei2016framework} derives knowledge of maritime traffic to detect low-likelihood behaviors and predict a vessel’s future positions using linear interpolation. Other linear interpolation methods are based on K Nearest Neighbor \cite{rancourt1999estimation} reputation where each point is predicted based on Constant Velocity Model (CVM) followed by nearest neighbor computation. Some studies \cite{miller1991modelling,kuijpers2009modeling} do bound an object's possible location by providing some deterministic methods (e.g., space-time prism model) and probabilistic methods \cite{winter2010directed}, but none of the approaches address abnormality within trajectory gaps. Figure \ref{fig:ComparisonofRelatedWork} shows a comparison with related work.

Our previous work defined an abnormal trajectory gap measure and provided a scalable algorithm Memo-AGD \cite{sharma2022abnormal} that bounds the entire range of an object's possible movements during the trajectory gap to derive possible anomaly hypotheses for human analysts. However, the previous approach was ineffective in handling spatial comparisons and computing abnormality scores. Here we extend this work by leveraging space-time indexing techniques to further reduce the number of required comparison operations. We also optimize merging operations to promote a large number of merged groups (Figure \ref{subfig:Intermediate Output}) and lower abnormal gap measure computations while preserving correctness and completeness. \\

\textbf{Contributions} : 
\begin{itemize}
\item We propose a Space-Time Aware Gap Detection (STAGD) with Dynamic Region Merge (DRM) approach where STAGD efficiently performs space-time indexing and merging of trajectory gaps, and DRM further improves the computations of gap abnormality scores.

\item We evaluate the proposed algorithms theoretically for correctness and completeness and also analyze the time complexity of the baseline and proposed algorithms.

\item We provide experimental evaluations of the proposed algorithms using the evaluation metrics such as computation time and accuracy under varying parameters where the proposed STAGD+DRM proves substantially faster than Memo-AGD.
\end{itemize}

\textbf{Scope:} In this work, we leverage the space-time prism idea for computing an abnormal gap measure and designing the proposed algorithms. For modeling gaps, we did not include acceleration, as this data was not available in the dataset. A detailed interpretation of acceleration using kinetic prism is discussed in Appendix \ref{Kinetic Prism:acceleration}. We are not considering low-density or sparse regions for detecting abnormal gaps. Other factors, such as device malfunction, hardware failure, and signal strength due to external factors (e.g., severe weather conditions), fall outside the scope of this work. In addition, we do not model anomalous trajectories without gaps. The proposed framework has multiple phases (i.e., filter, refinement, and calibration), but we limit this work to the filter phase. The refinement phase requires input from a human analyst and is not addressed here. Calibration of the cost model parameters also falls outside the scope of this work.

\textbf{Organization:} The paper is organized as follows: Section \ref{section:problemformulation} introduces basic concepts, the general framework, and the problem statement. Section \ref{section:baseline} provides an overview of the baseline approach AGD and Memo-AGD \cite{sharma2022abnormal}. Section \ref{section:SpaceTimeAware} describes the proposed Space-Time Aware Gap Detection (STAGD) and Dynamic Region Merge (DRM) algorithms, respectively, along with their execution traces. Theoretical analysis of both algorithms is provided in Section \ref{section:TheoreticalEvaluation}. Experiment design and results are reported in Section \ref{section:validation}. Section \ref{section:Discussion} discusses some of the discussions of our study. A study of related work is provided in Section \ref{section:Related_work}. Finally, Section \ref{section:Conclusion_future_work} provides the conclusion and future work.\\

\section{Abnormal Trajectory Gap Detection Problem}
\label{section:problemformulation}
In this section, first, we define key concepts and describe our general framework. Then, we formally define the problem, followed by some brief remarks about other ways the problem could be formulated. Finally, we present our problem formulation.


\subsection{Basic Concept}
\label{subsection:basicconcept}

\begin{definition}
A \textbf{study area} is a two-dimensional rectangular area in which the geo-locations of the input data are represented with latitude and longitude coordinate systems.
\end{definition}

\begin{definition}
A \textbf{spatial trajectory} is a trace of chronologically sorted GPS points in a series generated by a moving object in a geographic space.
\end{definition}
For example, $p_1$ → $p_2$ → ··· → $p_n$ represents a spatial trajectory, $p_{i} = (x_{i},y_{i},t_{i})$ where each point ($p_i$) is associated with a geographic coordinate set $(x_{i}, y_{i})$ and a time stamp $(t_{i})$.  Figure 1 shows the chronological sequence of GPS points (in grey), each associated with gap $G_{i}$.

\begin{definition}
An \textbf{object maximum speed ($S_{max}$)} is the maximum speed an object can attain based on the domain knowledge. The variable $S_{max}$ can be identified from publicly available maritime vessel databases \cite{aisdataUS}. For vehicles, humans, or animals, we can use the maximum physically allowed speed.
\end{definition}

\begin{definition}
An \textbf{effective missing period (EMP)} is a time period when a GPS signal is missing for longer than a certain threshold (e.g., 30 mins) which is externally specified by the end-user. Figure \ref{fig:problemstatement} shows that the EMPs for gaps $G_{1}$,..,$G_{n}$ are between the foci of the ellipses with greater than a certain interval (e.g., 30 mins, etc).
\end{definition}

\begin{definition}
A \textbf{signal coverage map (SCM)} is defined as a discretized grid space where each grid cell $GC_{i}$ represents the historically reported location traces $p_{i}$ generated by a set of trajectories.
\end{definition}

The process initiates by dividing the study area into a grid structure for a specific geographic region, ensuring global coverage of location traces. Initially, we calculate the maximum and minimum latitudes and longitudes to establish the minimum orthogonal bounding rectangle (MOBR). Subsequently, we evenly subdivide the MOBR using both latitude and longitude to construct grid cells and each grid cell is populated by the historic location-traces for a \textbf{fixed} time-frame. We finally classify each cell as reported (grey) or not reported (white) based on certain threshold $\theta$.

The signal coverage map classifies each grid cell denoted as $GC_{i}$ (e.g., $GC_{1}$, $GC_{2}$,..., $GC_{n}$) as reported or not. A reported cell ($GC_{m}$) is a cell whose total number of location traces is greater than a certain threshold $\theta$ (i.e., $\sum_{i=1}^{n} P_{i} \geq \theta$). In Figure \ref{fig:problemstatement}, the grid cells in \textit{grey} are reported cells or cells which are frequently reported by end-user, whereas the not reported cells are cells, where reported ship traffic is low or signal strength is under capacity (i.e., low density or sparse regions), are shown in \textit{white}. Hence, the signal coverage map is a boolean representation of discretized grid $\forall$ $GC_{i} \in \lbrack0,1\rbrack$.  It is important to note that the accuracy of the proposed algorithm significantly depends on the selection of the spatial region and the time frame. (more in Section \ref{section:Discussion}).

\begin{definition}
A \textbf{trajectory gap ($G_{i}$)} is defined as a space-time interpolated region within a missing location signal time period between two consecutive points.
\end{definition}
In Figure \ref{fig:problemstatement}, $G_{1}$,$G_{2}$,$G_{3}$...$G_{10}$ represent trajectory gaps that have been estimated via space-time interpolation in the form of geo-ellipses based on the properties of a space-time prism model \cite{miller1991modelling}. A space-time prism can be projected onto an $x-y$ plane \cite{pfoser1999capturing,miller1991modelling} in the form of a geo-ellipse which spatially delimits the extent of a moving object's mobility given a maximum speed during a missing signal period. Figure \ref{subfig:STPrism} shows a geo-ellipse representation with start point $P$ at ($x_1$,$y_1$) at time $t_1$ and end point $Q$ at ($x_2$,$y_2$) at time $t_2$, where $t_1 < t_2$ with focii $(x_1,y_1)$ and $(x_2,y_2)$.

\begin{figure}[htb]
    \centering 
\begin{subfigure}{0.27\textwidth}
  \includegraphics[width=0.8\linewidth]{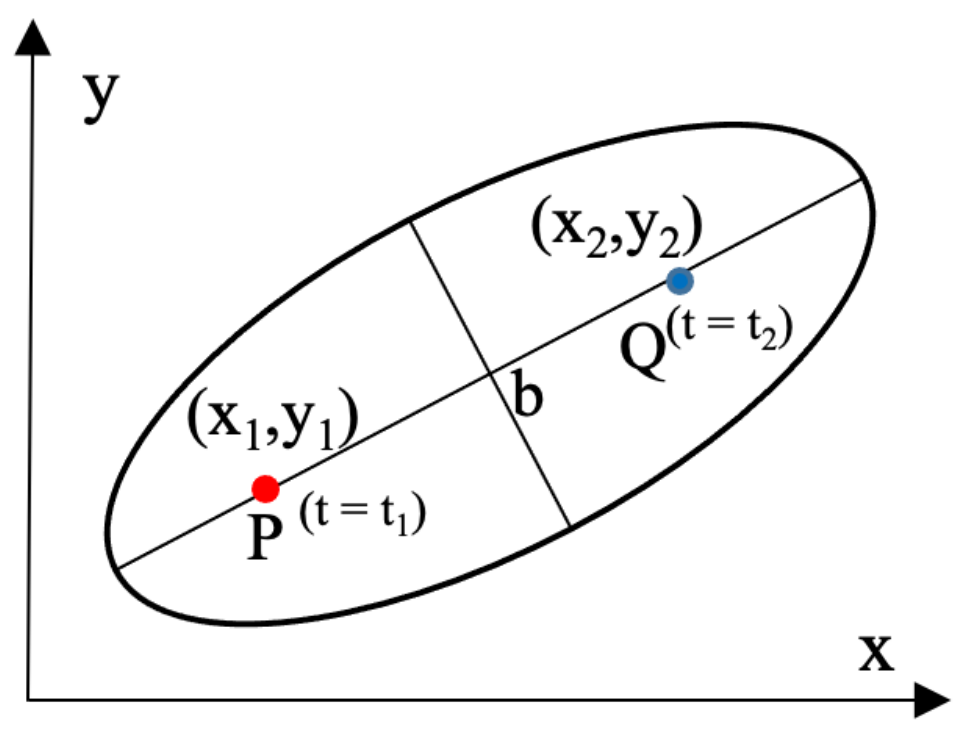}
  \centering
  \captionsetup{justification=centering}
  \caption{Projected Geo-Ellipse from Space-Time Prism}
  \label{subfig:STPrism}
\end{subfigure}\hfil 
\begin{subfigure}{0.27\textwidth}
  \includegraphics[width=0.9\linewidth]{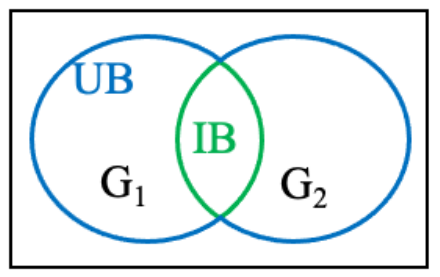}
  \captionsetup{justification=centering}
  \caption{IB (green) and UB (blue) of gaps $G_1$ and $G_2$.}
  \label{subfig:IntersectionUnion}
\end{subfigure}\hfil 
\centering
\captionsetup{justification=centering}
\label{Illstrative}
\caption{Example of (a) a geo-ellipse and (b) an intersection (IB) and union boundary (UB) (Best in Color)}
\end{figure}

The minimum orthogonal bounding rectangle (MOBR) is constructed over this ellipse, particularly focusing on the endpoints of its major axis by calculating the angle of this axis relative to a fixed coordinate system. This step is followed by rotating the ellipse to align this axis with one of the coordinate axes and then identifying the maximum and minimum extents of the ellipse in the direction perpendicular to the major axis. These extents, combined with the major axis endpoints, establish the MOBR's boundaries. Employing geometric transformations and the inherent properties of an ellipse, this approach is an effective means of determining the MOBR.

The R*-tree hierarchically organizes MOBRs within bounding boxes that cover child nodes allowing addition or removal of MOBRs without needing to restructure the entire tree. Through insertion, controlled splitting, and, when necessary, reinsertion, the tree maintains balance and enhances query performance. Incremental indexing of MOBRs in the R*-Tree, guided by a meticulous selection and splitting process, ensures efficient storage, access, and querying of spatial data, capitalizing on the R*-Tree's architectural advantages.

The temporal dimension, specifically the start and end points of a given period, is indexed separately within a large array, where each array index corresponds directly to a unique timestamp. This method ensures that every position in the array represents a distinct moment in time, allowing for efficient mapping and retrieval of temporal data. By aligning each index with a specific timestamp, this approach facilitates quick access to the associated start and end points, optimizing the process of querying temporal intervals within the dataset.

\begin{definition}
An \textbf{intersection and union boundary (IB and UB)} is the outline derived from a geo-ellipse region participating in intersection and/or union operations when two or more gaps have a spatial and temporal overlap. 
\end{definition}
For instance, in Figure \ref{subfig:IntersectionUnion}, two polygons, $G_{1}$ and $G_{2}$ undergo intersection (green outline) and union operations (blue outline). $G_{1}^{IB}$ or $G_{2}^{IB}$ is represented as Polygon($G_{1} \cap G_{2}$) whereas  $G_{1}^{UB}$ and $G_{2}^{UB}$is represented as Polygon($G_{1} \cup G_{2}$).

\begin{definition}
\label{def:AGM}
An \textbf{abnormal gap measure (AGM)} for a gap $G_{i}$ is the probability that a possible location of an object during a gap (unreported data time interval) has signal coverage. A higher value AGM indicates anomalous behavior since it means an object is not reporting its location despite having signal coverage.
\end{definition}
$GC_{int}$ denotes as the interpolated grid cells that reside within a spatial boundary (in the case of space time interpolation) or that intersect a  straight line (in the case of linear interpolation). The AGM computes the ratio of interpolated grid cells $GC_{int}$ intersecting with reported calls $GC_{m}$ (where $GC_{m}$ $\subseteq$ $GC_{int}$) to $GC_{int}$ . The formula is as follows:
\begin{equation}
\centering
   AGM = \frac{GC_{m} \cap GC_{int}}{GC_{int}}
\end{equation}

Figure \ref{subfig:LinearSpaceTimeComparison} (i) and (ii) show an example of computing AGM scores for linear and space-time interpolation, respectively. The linear interpolation captures only one $GC_{m}$ within its interpolated space $GC_{int}$ (i.e., a straight line). As a result, only 1 grey cell within 7 white cells is intersected by the straight line (i.e., $\frac{1}{7}$ or 0.14). By contrast, the space-time interpolation is able to capture 28 grey cells which intersect or are within its spatial boundary, resulting in $\frac{28}{32}$ or 0.80, and a better estimate than the linear interpolation. 

In this study, we confined our analysis to a static time frame spanning two years (2014-2016) as the computation of the abnormal gap measure is subject to significant variations over time. For example, employing a longer temporal range with a constant threshold, such as from 2009 to 2019, results in every cell being categorized as grey, in contrast to a more limited time frame, like from 2014 to 2016. This demonstrates that selecting an optimal threshold value is crucial for achieving a meaningful distribution of grey and white cells in the signal coverage map, which in turn facilitates the derivation of a significant AGM score within a trajectory gap.

\setlength{\belowcaptionskip}{-10pt}
\begin{figure}[htb]
    \centering 
\begin{subfigure}{0.48\textwidth}
  \includegraphics[width=\linewidth]{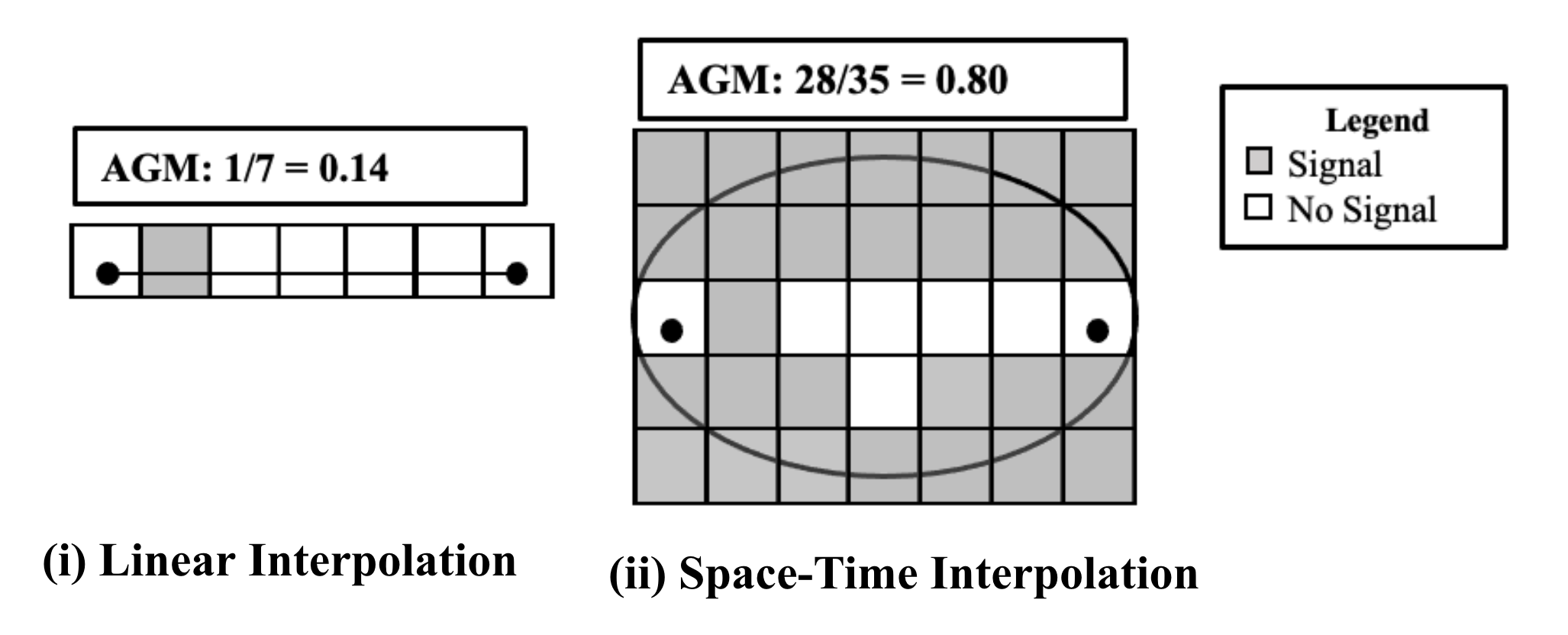}
  \centering
  \captionsetup{justification=centering}
  \caption{Abnormal gap measure}
  \label{subfig:LinearSpaceTimeComparison}
\end{subfigure}\hfil 
\begin{subfigure}{0.35\textwidth}
  \includegraphics[width=\linewidth]{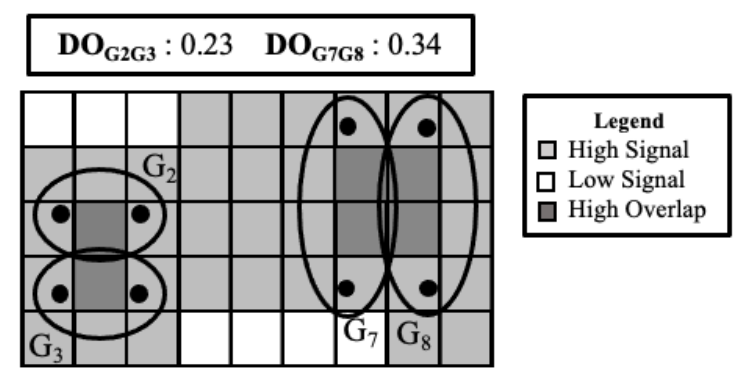}
  \captionsetup{justification=centering}
  \caption{Degree overlap threshold}
  \label{subfig:ParticipationIndex}
\end{subfigure}\hfil 
\centering
\captionsetup{justification=centering}
\label{fig:STPandIBUB}
\caption{Examples of abnormal gap measure (Left) and degree overlap threshold (Right)}
\end{figure}

\begin{definition}
\label{definition:PR}
A \textbf{degree overlap (DO)} is a measure to calculate the minimum of the ratio of interpolated grid cells $G_{i}$ that intersect with a signal coverage map within an intersection boundary (IB) of $G_{i}$ and $G_{j}$. It provides a measure to calculate the degree of participation for each gap $G_{i}$ within a gap set ($G_{1}$,$G_{2}$),  ($G_{3}$,$G_{4}$), and so on. The formula is as follows:
\end{definition}
\begin{equation}
\label{Equation2}
DO = min\lbrack\frac{{(GC_{m}^{G_{i}} \cap GC_{int}^{G_{i}})} \cap ({GC_{m}^{G_{j}} \cap GC_{int}^{G_{j}}})}{GC_{int}^{G_{i}}},\frac{{(GC_{m}^{G_{i}} \cap GC_{int}^{G_{i}})} \cap ({GC_{m}^{G_{j}} \cap GC_{int}^{G_{j}}})}{GC_{int}^{G_{j}}}\rbrack
\end{equation}
In this paper, we denoted the DO threshold by \textbf{$\lambda$}. The motivation behind the threshold is to provide early-stage filtering of a set of smaller coalesced gaps (e.g., 30 mins) interacting with relatively larger gaps (e.g., 3 hrs). This avoids the need for large coalescing operations, which result in more effective pruning of grid cells while calculating AGM scores with the signal coverage map. In addition, the length of the gap is independent of pixel sizes but directly proportional to the number of pixels (e.g., longer gaps tend to have larger spatial coverage, resulting in a large number of pixels in geo-ellipses. In addition, a higher lambda value provides a stricter condition, resulting in fewer merged groups (more details in Section \ref{subsection:Problemstatement})

Figure \ref{subfig:ParticipationIndex} shows two gap pairs $G_{2}$, $G_{3}$ and $G_{7}$, $G_{8}$ intersecting with some degree of overlap. According to Equation \ref{Equation2}, the degree of overlap DO$_{G2,G3}$ and DO$_{G7,G8}$ is 0.23 (i.e., min$\lbrack\frac{2}{9},\frac{2}{9}\rbrack$) and 0.34 (i.e., min$\lbrack\frac{4}{12},\frac{4}{12}\rbrack$) respectively, which means each gap participates equally in each pair.

\subsection{Framework}
\label{subsection:FrameWork}

Our aim is to identify possible abnormal regions for a given set of trajectory gaps and signal coverage area through a three-phase \emph{Filter} and \emph {Refine} approach. This paper introduces an intermediate filter phase to compute abnormal gaps and reduce the number of overlapping grid cells when there are pairs of trajectory gaps (Figure \ref{subfig:Intermediate Output}). Merging gaps will reduce the number of a high number of overlap grid cells (dark grey in Figure \ref{subfig:Intermediate Output}) and allow us to compute AGM scores for a larger number of merged groups (i.e., \textit{maximal group}) rather than individual trajectory gaps. This reduces the redundant work of considering high overlap region associated with each gap and further computing AGM scores, which is a time-intensive operation. 

The trajectory signals are first preprocessed to filter out trajectory gaps associated with multi-attributes (e.g., speed) to perform space-time interpolation. Then we apply the baseline Memo-AGD and proposed STAGD+DRM algorithms to optimize the number of spatial interactions by merging gaps, avoiding the need for redundant computations and improving computational efficiency. The output is a summary of significant abnormal gaps (Figure \ref{subfig:Final Output}) which helps a human analyst to scan a comparatively minimal area for inspection. The  results can be further verified via satellite imagery to derive a possible hypothesis about the anomaly (Figure \ref{fig:framework}).

\begin{figure}[ht]
    \centering
    \includegraphics[width=0.9\textwidth]{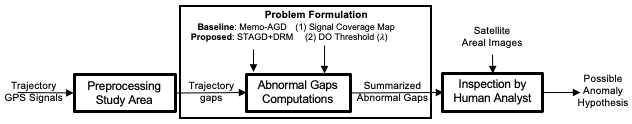}
    \caption{Framework for detecting possible abnormal gaps to reduce manual inspection by analyst\\}
    \label{fig:framework}
\end{figure}

\subsection{Problem Formulation}
\label{subsection:Problemstatement}

We formulated the problem to optimally identify an abnormal trajectory gap region in a spatiotemporal domain is formulated as follows:\\
\textbf{Input:}
\begin{enumerate}
	\item A study area $S$,
	\item A set of $|N|$ trajectory gaps
	\item A signal coverage map $SCP$
	\item A priority and a degree overlap threshold $\lambda$
\end{enumerate}
\textbf{Output:} Summarized abnormal trajectory gaps.\\
\textbf{Objective:} Solution Quality and Computational Efficiency\\
\textbf{Constraints:\\}
(1) Grid layout is not available.\\ (2) Acceleration is not available.\\ (3) Correctness and Completeness.

Figure~\ref{fig:problemstatement} (a) and (b) show the input as a two-dimensional representation of a signal coverage area and a set of trajectory gaps. Figure~\ref{fig:problemstatement} (c) and (d) show the significant summarized abnormal output after execution of the proposed algorithms.\\

\textbf{DO Threshold ($\lambda$) Interpretation:} The Degree Overlap (DO) threshold can be adjusted externally by a human analyst to yield more significant merged gap regions as shown in Figure \ref{fig:framework} where  serves as an input parameter for the proposed algorithm. Specifically, incrementing the value of imposes stricter criteria for merging gaps, allowing for more precision for human analysts to analyze individual gaps. The concept of Degree Overlap is founded on the principle of the minimum participation index between pairs of trajectory gaps or merged groups (collections of trajectory gaps). This metric provides insight into the potential proximity of two entities to engage in unauthorized activities in maritime domain (e.g., illegal oil transfers) despite adequate signal coverage. By adjusting the Degree Overlap threshold, analysts are afforded greater flexibility in scrutinizing each gap within merged groups. This includes the identification of significantly anomalous gaps within groups characterized by high anomaly scores. For example, a higher DO threshold  leads to a larger number of distinct abnormal gaps for verification by human analysts using satellite imagery, compared to lower threshold values. Nonetheless, verifying gaps through satellite imagery necessitates extensive post-processing work. Therefore, finding an optimal DO threshold is crucial to striking a balance between the quality of solutions and computational efficiency.

\subsection{The Choice of Problem Formulation}
\label{subsection:ChoiceProblemstatement}

The proposed abnormal gap measure (AGM) is based on identifying all possible occurrences of the vessels not reporting their location even though other moving vessels in the area did periodically. We $\textit{chose}$ to formulate the abnormal gap detection problem based on mandatory global maritime safety guidelines where ships must periodically report their geo-location to prevent potential collisions and illegal activities. The absence of location traces despite strong signal coverage makes it harder for authorities to disregard them and they are later sent to human analysts (Figure \ref{fig:framework}). 

However, the abnormality in ship behavior can be defined in other ways as well, such as monitoring the aggregated gaps based on the time of day or a ship’s physical and voyage-specific attributes. For instance, an increase in the total number of trajectory gaps at a specific time of day may alarm maritime authorities for specific vessel types (e.g., increased fishing efforts in night time is more suspicious as compared to day-time). Abnormality can also be defined with other physical attributes (e.g., headings, draft). For instance, draft, the depth of water that a ship needs in order to float, is inversely proportional to speed since loading of cargo reduces a ship’s speed but raises its draft and vice-versa. Another possibility is to study the distribution of a gap population to define more accurate probabilistic models such as Gaussian Processes to quantify potential locations within geo-ellipses further \cite{miller1991modelling}. However, each choice of problem formulation still lacks physical interoperability and requires ground-truth information, which later requires domain interpretation. More details are discussed in Section \ref{section:Discussion}.

\section{Baseline Approach}
\label{section:baseline}
Our previous work \cite{sharma2022abnormal} was based on a plane sweep approach for effective pruning of pairs of gaps in the filter phase and improving gap enumeration. The baseline method lays the groundwork for our proposed new work that will be explained later in Section \ref{section:SpaceTimeAware}. In this section, we provide a brief overview of the operations used in the baseline approach:


\textbf{Plane Sweep Approach:} We used a plane-sweep approach which is a filter and refine technique \cite{sharma2022analyzing} where the given study area is projected on a low-dimensional space (i.e., from 2D ellipses to 1D time-segments). In the filtering phase, all gaps are sorted based on x,y, or time coordinates to reduce storage and I/O cost. This allows the computation of intersections in a single pass. In the refinement phase, the AGD and Memo-AGD algorithm \cite{sharma2022abnormal} extracts gaps based on their spatiotemporal overlap. We used time as an initial filter phase to filter out gaps that are temporally synchronous. We perform a linear scan and check $t_{start}$ and $t_{end}$ for both gap pairs and check if $t_{start}$ and $t_{end}$ of the previous gap overlap with the $t_{start}$ and $t_{end}$ of the current gap.

\begin{figure}[ht]
    \centering
    \includegraphics[width=0.8\textwidth]{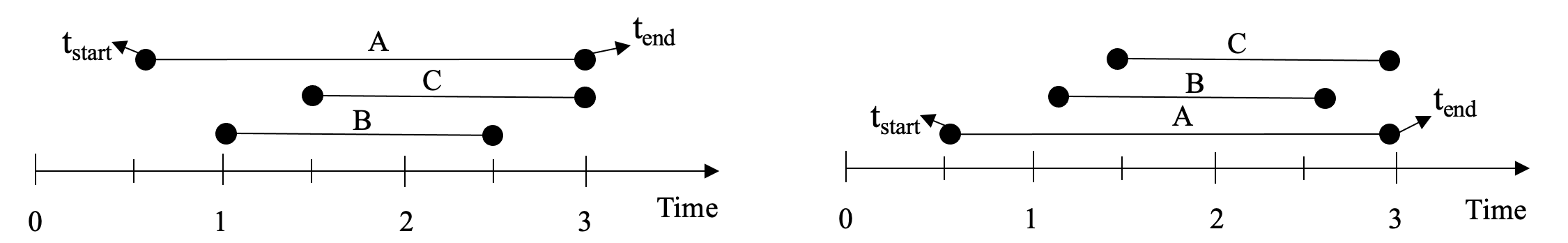}
    \caption{Plane Sweep Approach (Filter Step)\\}
    \label{fig:PlaneSweep}
\end{figure}

Figure \ref{fig:PlaneSweep} shows the initial filtering where we check $t_{start}$ and $t_{end}$ for gaps $A$, $B$, and $C$. Initial sorting prior to performing gap enumerations helps in considering linear scan via comparing $t_{start}$ and $t_{end}$. In the refinement step, imagine a plane sweeping through the start-time and end-time of each gap and check whether the start time of $B$ is smaller than the end time of $A$. If it is, then we consider A and B as potential candidates for the geo-ellipse spatial intersection.

\textbf{Trajectory Gap Enumeration:} After performing the plane sweep, the abnormal gap detection enumerates gaps by modeling every gap as geo-ellipse and applies an operation that checks which gaps intersect spatially, filtering out large gaps that don't meet the DO threshold $\lambda$. However, the operation is exorbitantly expensive because it considers all possible combinations of gap pairs. For instance, in Figure \ref{subfig:MemoAbnormalGapDetection}, since gaps $A$ and $B$ are intersected, gap $C$ also needs to satisfy the spatial intersection of $A$ and $B$ in order to be a part of the maximal group $A$, $B$ and $C$. Similarly, $D$ also need comparison with $A$, $B$, and $C$. This results in an exponential number of candidates for human analysts to post-process satellite imagery.


\setlength{\belowcaptionskip}{-10pt}
\begin{figure}[htb]
    \centering 
\begin{subfigure}{0.40\textwidth}
  \includegraphics[width=\linewidth]{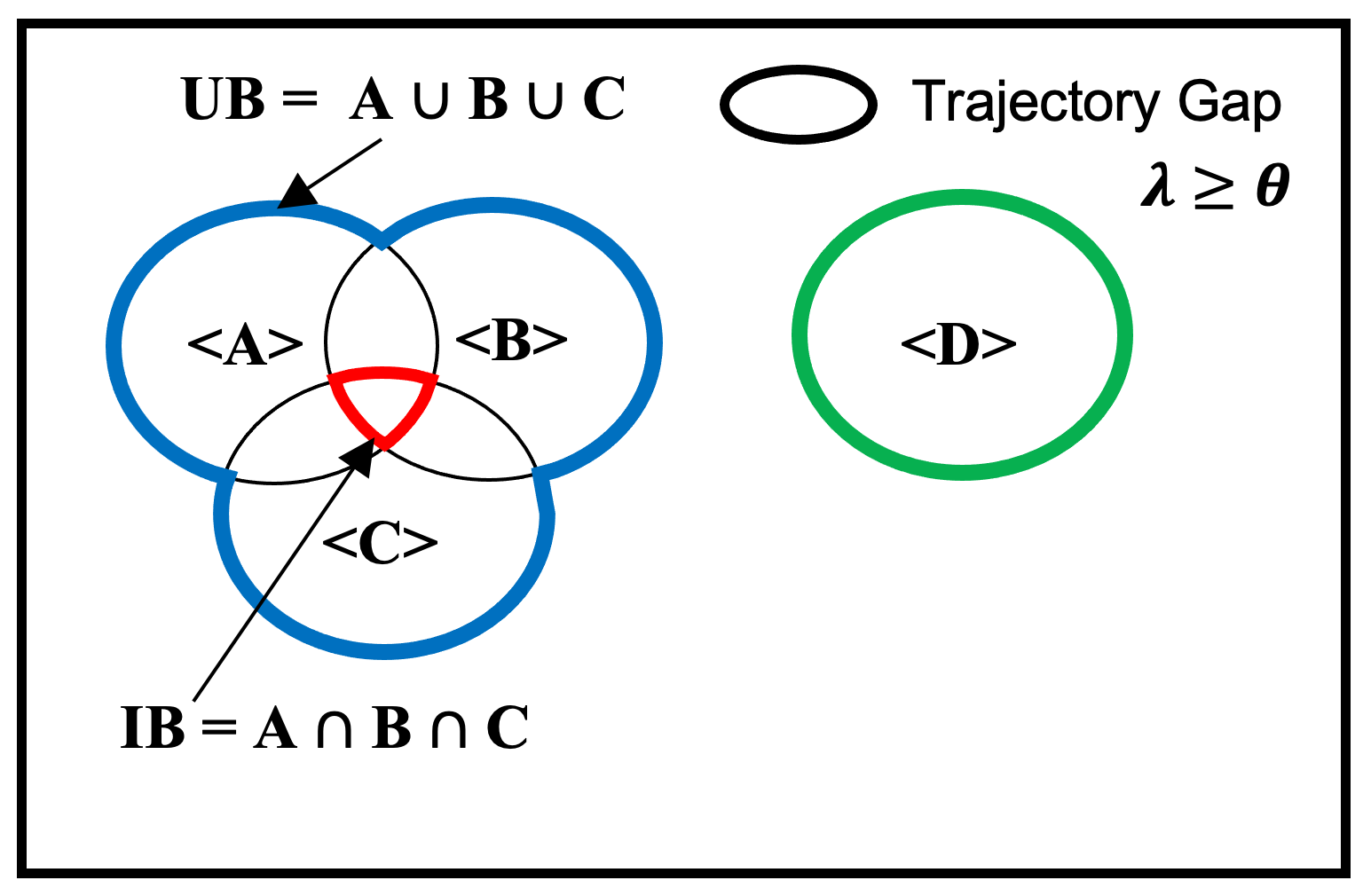}
  \centering
  \captionsetup{justification=centering}
  \caption{Memo-AGD Gap Enumeration}
  \label{subfig:MemoAbnormalGapDetection}
\end{subfigure}\hfil 
\begin{subfigure}{0.40\textwidth}
  \includegraphics[width=\linewidth]{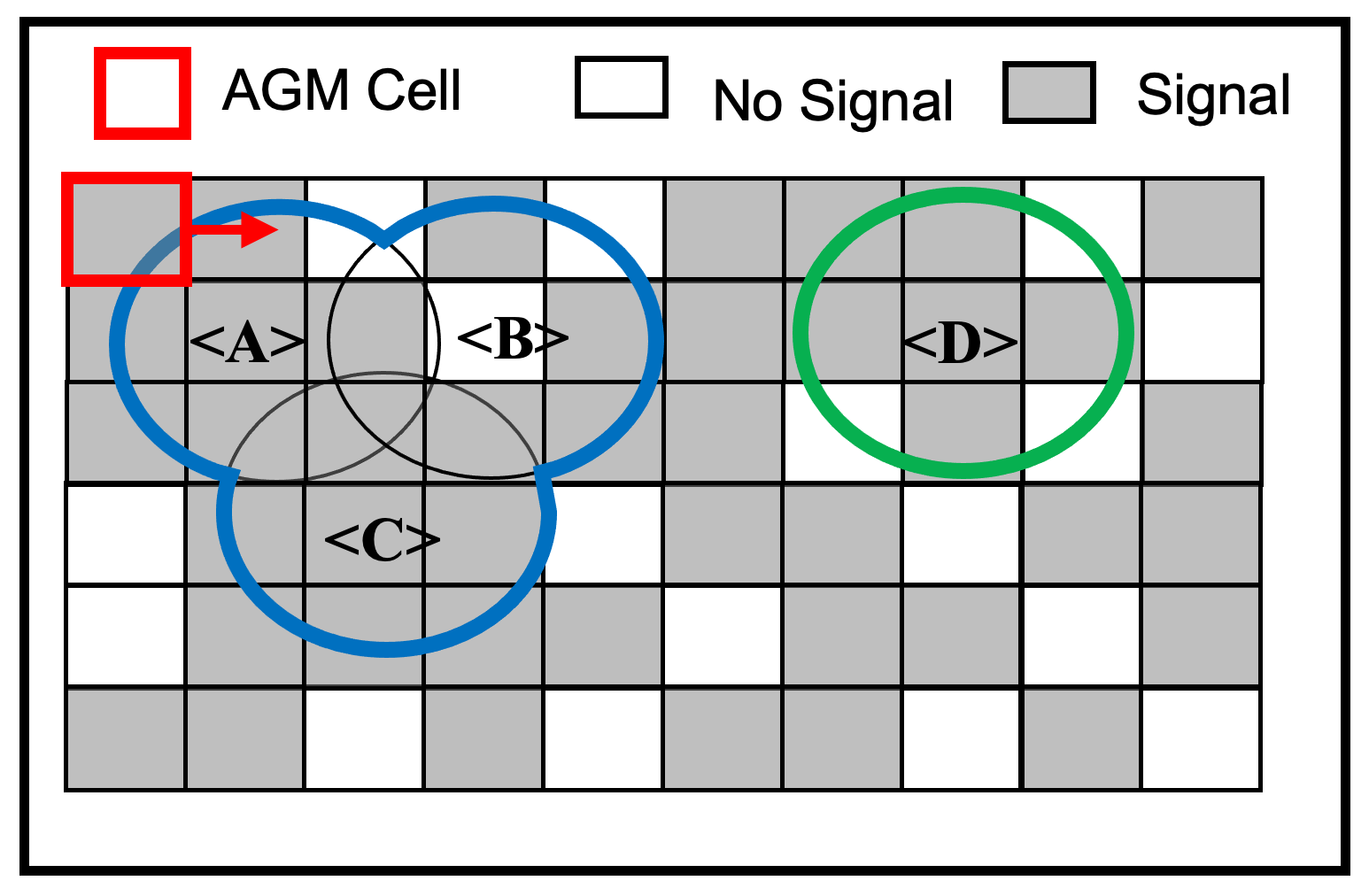}
  \captionsetup{justification=centering}
  \caption{AGM computations}
  \label{subfig:AbnormalGapMeasure}
\end{subfigure}\hfil 
\centering
\captionsetup{justification=centering}
\label{fig:AGDandMemoAGD}
\caption{Examples of (a) Memo-AGD Gap Enumeration and (b) Memo-AGD AGM computations}
\end{figure}

Memoized abnormal gap computation (Memo AGD) uses additional variables, namely, $G_{i}^{Obs}$ and $G_{i}^{LU}$, where $G_{i}^{Obs}$ keeps track of the total current elements in an Observed List and $G_{i}^{LU}$ provides a \textit{lookup table} which allow us to store information which was already involved in a prior intersection with $G_{i}$. \textit{Memoization} reduces the need to compute gaps unnecessarily. For instance, Polygon (A $\cap$ B $\cap$ C) in Figure \ref{subfig:MemoAbnormalGapDetection} can be cached so that computing gap $D$ does not require redundant computations of A, B, and, C resulting in an only comparison of Polygon ($A \cap B \cap C$). Appendix A provides a more detailed execution trace.


\textbf{Abnormal Gaps Measure (AGM) Computations:} After the gap enumeration, we merge all maximal subsets to reduce the total number of candidate pairs for reducing the redundant computation of cells in polygon operations in overlapping regions. The aim is to reduce the number of gap candidates while performing AGM computations which is typically a cell in polygon operation. For instance, Figure \ref{subfig:AbnormalGapMeasure} shows the gap <A>, <B>, and <C> are merged into Polygon ($A$ $\cup$ $B$ $\cup$ $C$), since $A \cap B \cap C$ $\neq$ 0. The merged regions then undergo \textit{cell-in} polygon operations to compute the abnormal gap measure (AGM) for Polygon ($A \cup B \cup C$). For instance, in Figure \ref{subfig:AbnormalGapMeasure}, AGM cells linearly perform a cell in polygon operation with Polygon ($A \cap B \cap C$) and Polygon($D$) to compute AGM scores of Polygon ($A \cap B \cap C$) as 0.55 and Polygon($D$) as 0.66.
\section{Proposed Approach}
\label{section:SpaceTimeAware}

Our baseline AGD and MemoAGD method performed certain computations (e.g., sorting, comparison), which can be further optimized by introducing a less-stricter merging condition to reduce redundant computations in real-time effectively. In addition, computing the AGM score is computationally expensive since it heavily depends on geographical grid computations and other factors (e.g., DO threshold) described in Section \ref{section:baseline}. Here, we introduce space-time aware gap detection (STAGD) to reduce redundant comparison operations and Dynamic Region Merge (DRM) to improve the efficiency of cells in polygon operation. In this section, we describe the proposed algorithm (STAGD+DRM) for enumerating gaps where each gap intersects in space-time and forms a merged cluster which is later merged dynamically to perform efficient AGM computations.

\subsection{Spatial-Time Aware Gap Detection (STAGD)}
\label{subsection:RTree}

First, we describe our conceptual understanding of the \textit{temporal} and \textit{spatial merge-aware indexing approach} with examples, which essentially handle the indexing of the ellipse by first handling the temporal indexing in the array list followed by hierarchical spatial indexing. After indexing, at the very end of the leaf node, we introduce the overlapping condition maximal union merge-based criterion along with its algorithm and execution trace. 

\subsubsection{Temporal Merge Aware Indexing}
\label{subsubsection:Time based indexing}

Traditional sorting algorithms (e.g., quicksort) have the interesting property of comparing and sorting elements in $\mathcal{O}(n\log{}n)$ time which proves to be computationally expensive, especially in case of large data volume. In this paper, we use \textit{comparison-less} sorting by assuming $n$ gap elements in the start time ($t_{start}$) range 1 to $k$ such that each gap can be indexed in an input array. We initiate our analysis by defining a global start time, denoted as $t_{start}$, and a global end time, denoted as $t_{end}$. These temporal bounds are subdivided into intervals of one second each, serving as indices within an array list. The value assigned to each index corresponds to $t_{start}$. Each time segment ($t_{start}$, $t_{end}$) is efficiently indexed via binary search, exhibiting O($nlogn$) time. Subsequently, we also maintain a monotonic stack with ($t_{start}$, $t_{end}$) elements sorted by $t_{start}$, efficiently comparing the $t_{end}$ of the current gap and the $t_{start}$ of the previous gap.

For instance, in Figure \ref{fig:TimeIndex}, $<A>$ to $<H>$ can be indexed based on $t_{start}$ where we can get a sorted order of trajectory gaps <C>,<B>,<A>,<D>,<E>, <F>, <H>, and <G>.  We can now perform a plane sweep to an already sorted list and compare the end-time and start-time of the gap, similar to the traditional Memo-AGD \cite{sharma2022abnormal}. Although this operation still requires quadratic computations, it will more efficiently pre-process time. 

\begin{figure}[ht]
    \centering
    \includegraphics[width=0.8\textwidth]{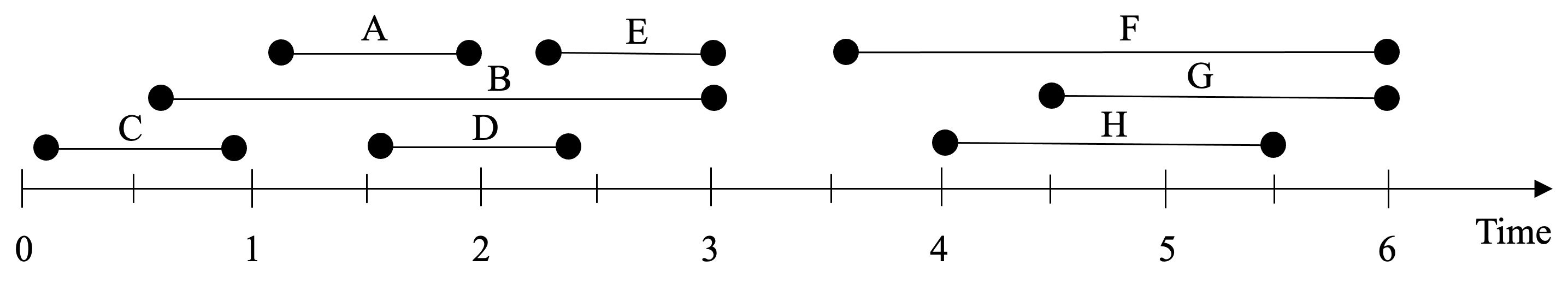}
    \caption{Temporal-Merge Aware Indexing of trajectory gaps\\}
    \label{fig:TimeIndex}
\end{figure}

\subsubsection{Spatial Merge Aware Indexing}
\label{subsubsection:Time based indexing}
A spatial index organizes access to data so that spatial objects can be found quickly without searching in a given spatial partition. It also enables the indexing of multidimensional objects, which can drastically speed up GIS operations like intersections and joins. For instance, R-Trees organize data in a tree-shaped structure, with a minimum orthogonal bounding rectangle (MOBR) at the nodes. MOBRs indicate the farthest extent of the data and can be indexed easily but they can lead to "false positives" (i.e., MOBRs may intersect when the precise geometries (geo-ellipses) does not actually intersect). In this paper, we leverage this space-time access method to efficiently compare gaps that are further spatially proximate effectively.

\textbf{Index Generation: } We first calculate the minimum and maximum values for xmin, ymin, xmax, and ymax across all MBRs to create a global rectangle that encompasses all individual MBRs. This effectively sets the stage for a hierarchical organization in which each node represents a spatial subdivision, thereby enhancing query performance through spatial locality. We have found that operations such as node splitting, bounding box updates, and overflow handling are better managed by estimating the overall spatial extent of the study area. Moreover, this approach efficiently handles a wide range of queries, including point queries, range queries, and nearest neighbor searches, by exploiting the spatial hierarchy and locality inherent in the global rectangle and its subdivisions.

\textbf{Insertion:}After checking for temporal overlaps, we check whether the geo-ellipse of two objects intersects or not. However, certain geo-ellipses are not involved in any groups (e.g., <F>, <G>, and <H> in Figure \ref{fig:SpaceIndex}), resulting in high computational cost. For instance, in Figure \ref{fig:SpaceIndex}, <A>, <B>, and <C> intersect in space and time, but <F>, <G>, and <H> are spatially distant. Hence, future gaps which are closer to <F>, <G>, and <H> but distant from <A>, <B>, and <C> will only perform intersections within their defined spatial partition (i.e., $L_{1}$). Hence, indexing such isolated gaps using hierarchical data structures may reduce quadratic comparisons.

\begin{figure}[ht]
    \centering
    \includegraphics[width=0.8\textwidth]{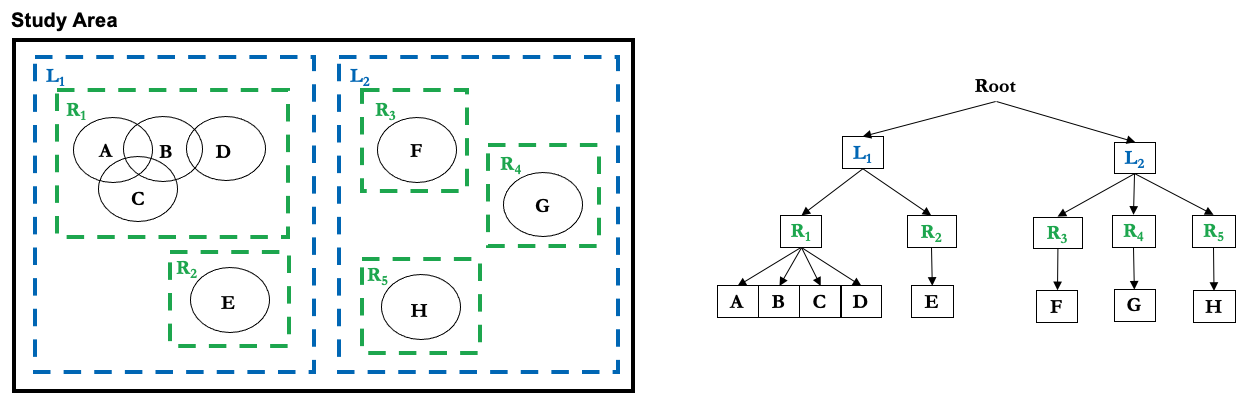}
    \caption{Spatial Merge Aware Indexing of trajectory gaps\\}
    \label{fig:SpaceIndex}
\end{figure}

Figure \ref{fig:SpaceIndex} shows gaps <A> to <H> spatially indexed at various hierarchical levels $L_{1}$ and $L_{2}$ where gaps <A> to <E> are indexed in $L_{1}$ and <F> to <H>, which are isolated and indexed in $L_{2}$. Suppose gap <H> needs to satisfy the criteria of spatiotemporal joins, which can be avoided by the comparison of the maximal subset of $A \cap B \cap C$ and $D \cap E$ by indexing at $L_{2}$. Then <H> can only be compared with <F> and <G> in addition to logarithmic levels and reducing the comparison operations from quadratic to logarithmic.

In addition, maximal intersection criteria result in multiple isolated (or independent) gaps, which are costly to compare and compute abnormal gap measures for future gap intersection operations. For instance, in Figure 11, <D> is partially intersecting with the maximal set <A,B,C>, then it can also be included in the maximal set to reduce isolated gaps as discussed in the following section.

\subsubsection{Maximal Union Merge-based Criteria}
\label{subsubsection:Time based indexing}

We first index the gaps and then determine if they intersect with the largest merged unit (maximal union) within the merged group. If they do, they are then merged with the group, but if not, they are considered isolated (or independent) gaps. This leads to a large number of isolated gaps, even in denser regions, which makes it difficult to perform spatial comparison operations and AGM computations. In addition, if all gaps are stored in a single index, leading toward high space and time complexity. This can negatively impact spatial or spatiotemporal indexing methods (e.g., 3D R-Trees). To overcome this issue, we first perform spatiotemporal indexing followed by proposing a \textit{maximal union merge-based criterion} to form maximum contiguous merged regions by using the union and intersection properties of gaps while taking into account the DO threshold value. More details in Algorithm \ref{STAGD}.

\begin{algorithm}
\caption{Space-Time Aware Indexing Algorithm}
\label{STAGD}
\footnotesize
\begin{flushleft}
\hspace*{\algorithmicindent}\textbf{Input:}\\
\hspace*{\algorithmicindent} A Study Area $S$ and A Set of Gaps \lbrack$G_{1}$,...,$G_{n}$\rbrack\\
\hspace*{\algorithmicindent} An Estimate Signal Coverage Area $MP$\\
\hspace*{\algorithmicindent} Participation Index $\leftarrow \lambda$ and Significant Score $\leftarrow K$\\
\hspace*{\algorithmicindent}\textbf{Output:}\\
\hspace*{\algorithmicindent} Abnormal Gaps with Top $K$ AGM\lbrack$G_{1}$,..,$G_{k}$\rbrack
\footnotesize
\end{flushleft}
\begin{algorithmic}[1]
    \Procedure{:}{}
    \State \textbf{Step 1:} Initialize AGM list $\leftarrow \emptyset$, Polygon($G_{i}$), and Non Observed List $\leftarrow \emptyset$
    \State Temporal Index List $\leftarrow \lbrack min(t_{start}),max(t_{end}) \rbrack$ $\forall$ gaps $\in$ $G_{i}$
    \State \textbf{Step 2:} Perform temporal indexing and check overlap $\forall$ $\lbrack t_{start},t_{end} \rbrack$ $\in$ $G_{i}^{EMP}$
    \State Index $t_{start} \in $ $G_{i}^{EMP}$ in Temporal Index List
    \State Insert $G_{i}$ in Non Observed List to preserve \textit{monotonic property}
    \ForEach{$G_{i}$ $\in$ Non Observed List}
        \State $G_{j}$ $\leftarrow$ topmost element from local observed list
        \State Local Observed List $\leftarrow$ Global Observed List
        \While {Local Observed List $\neq \emptyset$}
            \If{$G_{j}^{t_{end}}$ $\leq G_{i}^{t_{start}}$ where $\lbrack t_{start},t_{end} \rbrack$ $\in$ both $G_{i}$ and $G_{j}$}
                \State Add $G_{j}$ in potential candidate gap pairs $\in$ $G_{i}^{CP}$
            \EndIf
            \State Remove $G_{i}^{CP}$ from Local Observed List
        \EndWhile
        Add $G_{i}$ to Global Observed List
        \State \textbf{Step 3:} \textbf{Perform Spatial Indexing and Check Maximal Union Criteria}
        \State Perform Hierarchical Spatial Indexing
        \ForEach{$G_{j}$ $\in$ candidate pairs of $G_{i}$ $\in$ $G_{i}^{CP}$}
            \If{Polygon($G_{i}$) $\cap$ Union and Intersection Group of $G_{j}$ and DO $\geq \lambda$}
                \State Merge $G_{i}$ over union group $\in$ $G_{j}$ and update both $G_{i}^{IB}$ and $G_{j}^{UB}$
                \State Terminate and skip towards next iteration
            \EndIf
            \If{Polygon($G_{i}$) $\cap$ Union Group of $G_{j}$ and DO $\geq \lambda$}
                \State Merge $G_{i}$ over Union group $\in$ $G_{j}$ and update $G_{j}^{UB}$
            \EndIf
        \EndFor
    \EndFor
    \State \textbf{Step 4:} \textbf{Compute AGM Score $\forall$ $G_{i}$ in Observed List}
        \ForEach{$G_{j}$ $\in$ Observed List}
            \If{$G_{j}^{count}$ = Number of gaps in the subset <$G_{j}$>}
                \State AGM List $\leftarrow$ $AGM(MP \cap G_{j})$ 
            \EndIf
        \EndFor
    \State \textbf{Step 5:} \textbf{Return K Significant Abnormal Gaps}
    \State \Return Top K gaps \lbrack$G_{1}$,..,$G_{k}$\rbrack based on AGM Score
\EndProcedure
\end{algorithmic}
\end{algorithm}

\textbf{Step 1:} First we initialize the AGM list and Non-Observed List to model every gap $G_{i}$ as a geo-ellipse (i.e., Polygon($G_{i}$)) along with an EMP (Effective Missing Period) based on $G_{i}^{EMP}$ = \lbrack ${G_{i}^{t_{start}}}$,${G_{i}^{t_{end}}}$\rbrack. We further initialize the Temporal index list $\in$ [min(${t_{start}}$), max(${t_{end}}$)] as a filtering step for performing the linear scan, which results in avoiding comparison-less operations.

\textbf{Step 2:} For each  ${G_{i}^{EMP}}$ $\in$ $\lbrack t_{start},t_{end} \rbrack$, we leverage $t_{start}$ to index the gap's position in the temporal index list. We then perform a linear scan to gather all the gaps $G_{i}$ and save them in a Non-observed list to preserve \textit{montonicity} based on the increasing order of the time-stamp. For each $G_{i}$ in the Non-Observed list, we compare the DO threshold ($\lambda$) with the topmost gap in the Local Observed List and if the condition meets the threshold, then we add the pair to a list of candidate pairs of $G_{i}^{CP}$ $\in$ [$G_{i}$,$G_{j}$,...,$G_{n}$]. These candidate pairs in $G_{i}^{CP}$ are then spatially indexed using widely used indexing (e.g., R-Trees) or any other hierarchical spatial partitioning method.

\textbf{Step 3:} Next, we perform spatial intersections of geo-ellipse based on the \textit{maximal union merge criterion} on the existing groups for gaps that reside within the leaf node and satisfy DO threshold condition and within the same partition at the leaf node level of the hierarchical spatial index. For each $G_{i}$, we then compare each gap $G_{j}$ in $G_{i}^{CP}$ (candidate pairs) and check the \textit{maximal union merge criterion} based on the \textit{Intersection} and \textit{Union} group. The criterion states that if gap $G$ is not intersecting with intersection bounds ($IB$) but is intersecting with the union bound (UB) of the maximal group, we then merge the gap pair with the union of the maximal group. For instance, the pseudo-code related to lines 17-22 shows the maximal union criterion where we first check if each gap $G_{i}$ is intersecting with both Union and Intersection of the maximal group. If it is, we first merge, then terminate and move on to the next candidate pairs, otherwise we further check if gap $G$ is intersecting with the Union of the maximal group. If it is, we merge all participating gaps and create MOBRs for effective AGM computations. This provides a stricter condition to make maximal groups and lower the count of isolated gaps.

\textbf{Step 4:} For each gap $G_{i}$ $\in$ maximal union group, we perform AGM computation via linear scan operation of the maximal union group (e.g., $A \cup B \cup C$), intersection group (e.g., $A \cap B \cap C$), and isolated gaps (e.f., $A$) intersection with the signal coverage map SCM and subsequently save all merged gaps in the Observed List.

\textbf{Step 5:} Extract abnormal gaps based on the AGM Score.

\textbf{Execution Trace:} Figure \ref{fig:Step1-2} shows the execution trace of steps 1-2 of Algorithm \ref{STAGD}. Step 1 initializes gaps <A>, <B>, <C>, <D>, and <E> with a pre-defined EMP $\geq$ $\theta$ along with additional \textit{apriori} variables as defined in Memo-AGD \cite{sharma2022abnormal}, namely $G^{CP}$ $G^{UB}$,$G^{IB}$ and  Non-Observed List is initialized as $\emptyset$ and Temporal Index List $\in$ [min($t_{start}$),max($t_{end}$)] of total gaps as shown in Figure \ref{fig:Step1-2} (e.g., [$t_{start}$ $\in$ A,$t_{end}$ $\in$ E]). For \textit{comparison-less} sorting, step 2 performs a linear scan towards the $t_{start}$ of each gap in Temporal Index List for populating the Non-Observed List as shown in Figure \ref{fig:Step1-2}. Finally, we check if the temporal overlap is satisfied by comparing the previous gap start point ($t_{start}$) to that of the current gap's endpoint ($t_{end}$) using a temporary data structure. In Figure \ref{fig:Step1-2}, we used \textit{stack} to compare prior elements with current elements from the Local Observed List and \textit{cache} them in $G^{CP}_{i}$ if the temporal condition is satisfied. For instance at $i$ = 1, $t_{end}$ $\in$ <B> is compared to $t_{start}$ $\in$ <A> in stack and <A> is stored in $G^{CP}$ $\in$ <B> (resulting in $G_{i}^{CP}$ $\in$ <A,B> and $G_{j}^{CP}$ $\in$ <A,B>) since it satisfies temporal overlap condition. Similarly, for each gap, we store $G^{CP}$ for the spatial comparison at the leaf level of the spatial index.



\begin{figure}[ht]
    \centering
    \includegraphics[width=1.0\textwidth]{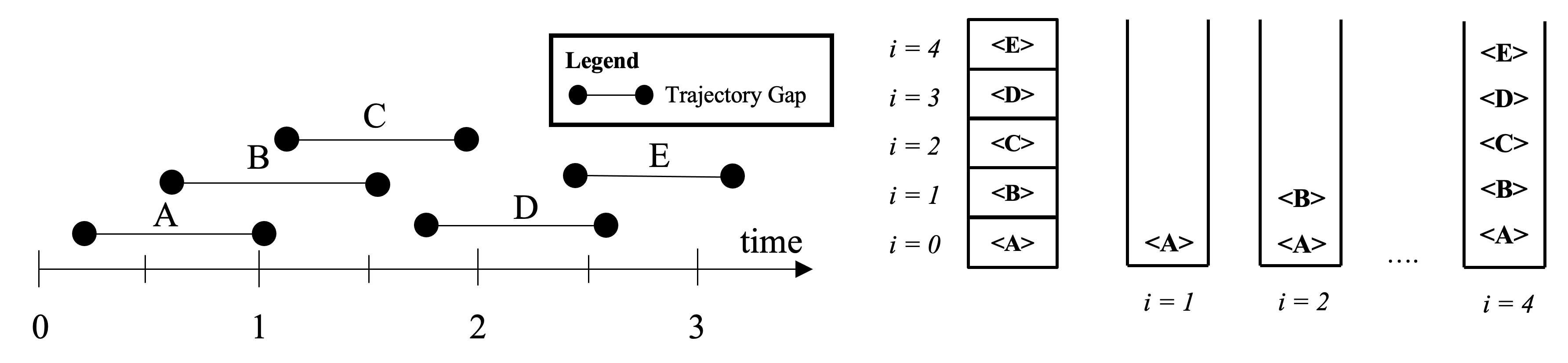}
    \caption{Spatial Merge-Aware with Maximal Union Merge Criteria Execution Trace (Step 1-2)\\}
    \label{fig:Step1-2}
\end{figure}

Figure \ref{fig:Step3} illustrates the execution trace of the maximal union merge criterion in Algorithm \ref{STAGD}. After hierarchical indexing as explained in Figure \ref{fig:SpaceIndex}, we focus on gaps <A> to <E> which are residing in the same index with candidate pairs $G_{i}^{CP}$ and then intersected with $IB$ and $UB$. For instance, gaps <B> and <C> satisfy $IB$ with $\lambda$ $\geq$ $\theta$. However, the $G_{i}^{CP}$ of <D> only intersects with UB of <B> (i.e., <$A \cup B \cup C$>). If <D> $\cap$ $A \cup B \cup C$ is greater than DO threshold $\lambda$, then it will still be merged to A $\cup$ B $\cup$ C $\cup$ D. Similarly if  $G_{i}^{CP}$ of <E> does not spatially intersect with <D> and the algorithm terminates.

\begin{figure}[ht]
    \centering
    \includegraphics[width=1.0\textwidth]{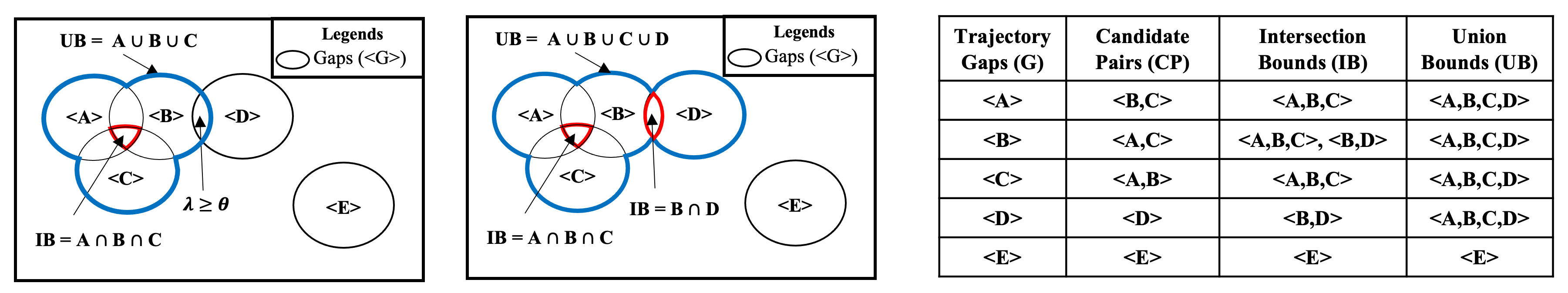}
    \caption{Spatial merge-aware with maximal union merge criteria execution trace (Step 3)\\}
    \label{fig:Step3}
\end{figure}

\subsection{Dynamic Region Merge (DRM) Approach}
\label{subsection:DRM}
Our method of caching temporal information and spatial indexing effectively reduces comparison operations prior to performing AGM computations. In addition, maximal union merge criteria reduce the total number of merged gap unions to aid cells in polygon interactions for the AGM computations. However, the maximal union criterion relies on a DO threshold and the nature of the data distribution, i.e., whether the distribution of gaps is uniform or skewed. For instance, if a gap distribution is skewed and no gaps are involved in any spatiotemporal intersection operations, then all gaps will reside in a single index. This results in a large number of isolated gaps which also increases cell in polygon operations as shown in Figure \ref{subfig:AbnormalGapMeasure}. Hence, we propose more optimal two methods for finding significant AGM scores based on maximal groups considering the sliding window approach and early termination criterion.

\subsubsection{Sliding Window with Memoization Approach}
\label{subsubsection:SlidingWindow}

Given multiple gaps intersecting within a spatial partition, we compute AGM scores incrementally whenever a new gap intersects a maximal group spatially. In this technique, we cache the previously computed AGM cells for each maximal group defined by the maximal union merge criterion, which permits intersection and union criteria thresholds. Hence, we only need to append the reported signal coverage cells associated with the new gap and avoid the need to recompute the AGM scores of the entire maximal subgroup. The sliding window technique can achieve this, along with caching current cells involved with the group in a linear time (similar to the linear scan) but caching ensures fewer cells in polygon comparisons. An example of the approach is illustrated in Figure \ref{fig:Sliding} as follows:

\begin{figure}[ht]
    \centering
    \includegraphics[width=0.95\textwidth]{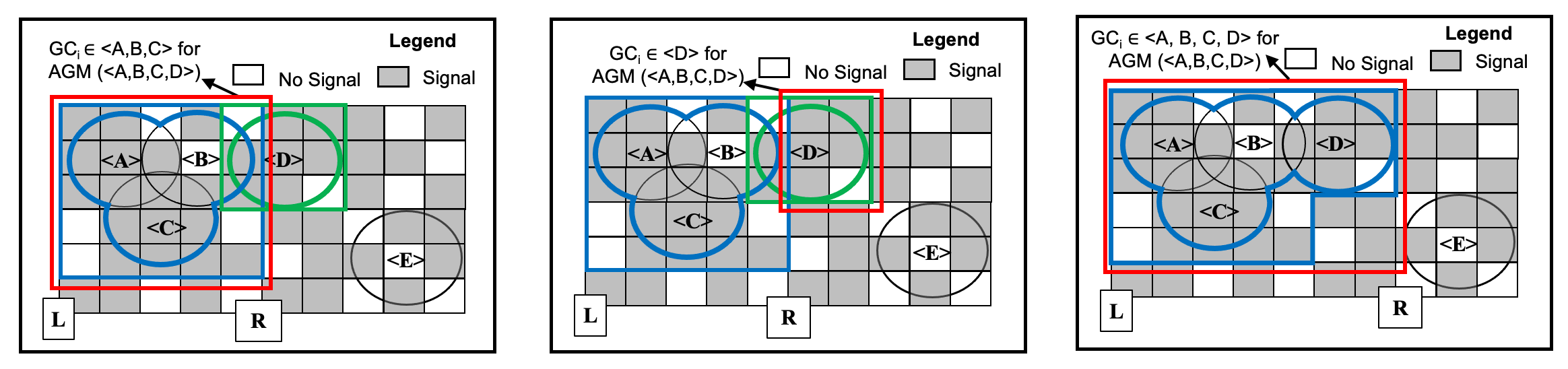}
    \caption{Proposed sliding window approach to compute AGM incrementally\\}
    \label{fig:Sliding}
\end{figure}

Figure \ref{fig:Sliding} shows an illustration of a sliding window with caching where gaps <A>, <B>, and <C> are involved in maximal subsets and we further initialize left (L) and right (R) pointer around the MOBR defined around Polygon( <A,B,C>). AGM score is based on the cells involved in the MOBR defined around them. For instance, if gap <D> intersects the $A \cap B \cap C$ above the given DO threshold ($\lambda$), we can merge gap <$D$> with the <$A \cup B \cup C$> by itself based on \textit{maximal union merge-based criteria}, then we are required to recompute scores for <$A \cup B \cup C$> (leftmost Figure \ref{fig:Sliding}). To eliminate this redundancy, we cache such cells $\in$ <$A,B,C$> and append only $GC_{i}$ $\in$ <$D$> with $GC_{i}$ $\in$ <$A,B,C$>.

\textbf{Incremental Spatial Indexing:} The minimum orthogonal bounding rectangle (MOBR) of the group <A, B, C> is defined by the coordinates ($x_{1}, y_{1}$), ($x_{2}, y_{1}$), ($x_{1}, y_{2}$), and ($x_{2}, y_{2}$), with the left pointer (L) encompassed within the range [($x_{1}, y_{1}$), ($x_{1}, y_{2}$)] and the right pointer (R) within [($x_{2}, y_{1}$), ($x_{2}, y_{2}$)]. Upon encountering an intersecting MOBR of <D>, we perform union of <A, B, C> and <D> and updating the smallest ($x_{min}$,$y_{min}$) and largest ($x_{max}$,$y_{max}$) coordinates among the two MOBRs resulting with <A,B,C,D>. In case of union, we update the right pointer (R) to a new [($x_{2}^{new}, y_{1}^{new}$), ($x_{2}^{new}, y_{2}^{new}$)], else we update the Left pointer (L) [($x_{1}^{new}, y_{1}^{new}$), ($x_{1}^{new}, y_{2}^{new}$)]. We index grid cells situated between the left and right pointers and update the AGM score for the merged group.

\subsubsection{Early Termination Criteria:}
\label{subsubsection:Caching}

Since the proposed criterion helps to reduce isolated gaps by providing less stricter conditions by merging with the maximal subgroups via union operations, it may also result in a long chain of gaps which may also lead to false positive or false negative results based on AGM criteria. For instance, if there is some variability in terms of signal coverage map (as shown in Figure \ref{fig:Exec_Trace_DRM} (b) and Figure \ref{fig:Exec_Trace_DRM} (c)) and a new gap (e.g., <$D$>) may fall below the defined threshold and the entire maximal subgroup may not be abnormal. To avoid such false negatives, we use an early termination criterion where we incrementally check both the previous and new AGM scores prior to satisfying the maximal union merge criteria as demonstrated in Equation \ref{earlytermination} below: 

\begin{equation}
\label{earlytermination}
    || AGM (<G_{i},..G_{k}>) - AGM(<G_{j}>) || \geq \delta
\end{equation}

where $\delta$ is the threshold for determining whether the difference between previous and new AGM scores is more significant or not. If it's not, we then we merge <$G_{i},..G_{k}$> and <$G_{j}$> and compute AGM (<$G_{i},..G_{k}$> $\cup$ <$G_{j}$>). Otherwise, we calculate the AGM of <$G_{i},..G_{k}$> and <$G_{j}$> separately i.e., AGM (<$G_{i}$,..$G_{k}$>) and AGM(<$G_{j}$>).

\subsubsection{AGM-based Decreasing Queue:} 
\label{subsubsection:Caching}

Given multiple merged regions in some leaf node, we can further leverage the $\textit{monotonic property}$ as mentioned while indexing the temporal gaps start-times which results in $\textit{comparison-less}$ sorting. Here we leverage a decreasing queue (deque) data structure, which maintains the non-increasing order of AGM scores of the merged-groups. The main intuition behind this is doing an on-the-fly comparison of the AGM score of incoming gap $G_{i}$ with the AGMs of gaps which are previously sorted in the decreasing order such that Equation \ref{earlytermination} terminates the comparison at an early stage. For instance, a gap $G_{i}$ with a high AGM score is likely to merge with groups having a high AGM score and vice-versa such that Equation \ref{earlytermination} is satisfied.

\begin{figure}[ht]
    \centering
    \includegraphics[width=0.80\textwidth]{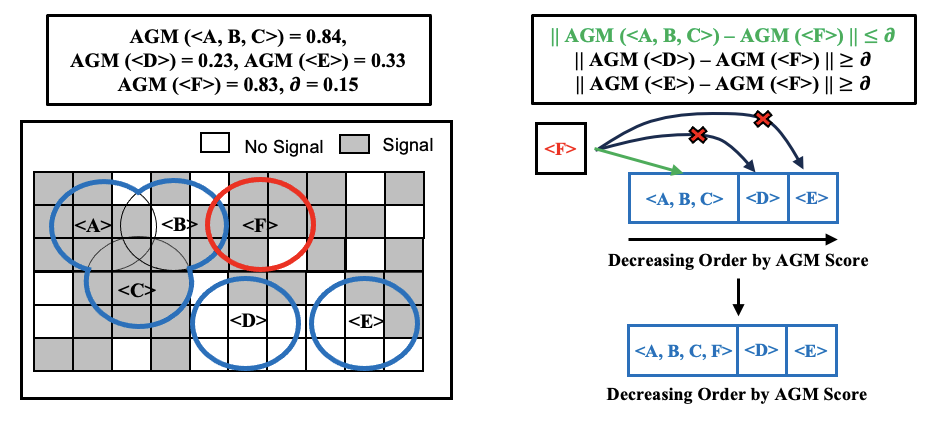}
    \caption{Proposed Decreasing Queue strategy to sort gaps in leaf node by decreasing order\\}
    \label{fig:Deque}
\end{figure}

Figure \ref{fig:Deque} illustrates the decreasing queue approach, where the AGM scores for merged gaps <A, B, C> for gap <D>, and for gap <E> are 0.84, 0.23, and 0.33 respectively. Consider <F> undergoing spatial comparison with the groups <A, B, C>, <D>, and <E>. With a specified threshold $\delta = 0.15$, it becomes evident that <F>, having an AGM of 0.83, is more likely to merge with the group <A, B, C> than with <D> or <E>, owing to the similarity in AGM values. Organizing all minimum orthogonal bounding rectangle (MOBR) coordinates (xmin, ymin, xmax, ymax) in a decreasing queue—sorted by descending AGM values—facilitates the early termination of comparisons. For instance, <F> need only compare with <A, B, C> and can terminate the process upon reaching <D>, as $||AGM(<A, B, C>) - AGM(<F>)|| \geq \delta$. This reduction progresses from $n$ possible comparisons to $k$ (where $k < n$), thus enhancing computational efficiency.

Algorithm \ref{DRM} explains in detail.

\begin{algorithm}
\caption{Dynamic Region Merge Algorithm}
\label{DRM}

\footnotesize
\begin{flushleft}
\hspace*{\algorithmicindent}\textbf{Input:}\\
\hspace*{\algorithmicindent} A Study Area $S$ and A Set of Gaps \lbrack$G_{1}$,...,$G_{n}$\rbrack\\
\hspace*{\algorithmicindent} An Estimate Signal Coverage Area $MP$\\
\hspace*{\algorithmicindent} Participation Index $\leftarrow \lambda$ and Significant Score $\leftarrow K$\\
\hspace*{\algorithmicindent}\textbf{Output:}\\
\hspace*{\algorithmicindent} Abnormal Gaps with Top $K$ AGM\lbrack$G_{1}$,..,$G_{k}$\rbrack
\footnotesize
\end{flushleft}
\begin{algorithmic}[1]
    \Procedure{:}{}
    \State \textbf{Step 1 to Step 3 are same as Algorithm 1}
    \State \textbf{Step 4:} \textbf{Compute AGM Score $\forall$ $G_{i}$ in Observed List}
        \ForEach{$G_{i}$ $\in$ Spatial Index}
        \State Perform sliding window increment to add $GC_{i}$ $\in$ $G_{i}$
        \State AGM ($G_{i}$) $\leftarrow$ AGM($GC_{i}$ $\in$ $G_{i}$)
        \State AGM ($G_{i} \cap$ Union Group) $\leftarrow$ AGM ($GC_{i}$ $\in$ Union Group + $GC_{i}$ $\in$ $G_{i}$)
        \If {$G_{i}$ $\in$ Union and/or Intersection Group}
            \If {AGM($G_{i} \cap$ Union Group) - AGM($G_{i}$) $\leq$ $\delta$}
                \State \textbf{AGM Decreasing Queue} $\leftarrow$ AGM($G_{i} \cap$ Union Group)
            \Else { Compute separate AGM for $G_{i} \cap$ Union Group and $G_{i}$}
            \EndIf
        \EndIf
        \EndFor
    \State \textbf{Step 5:} \textbf{Return Abnormal Gaps to the Human Analysts}
    \State \Return \lbrack$G_{1}$,..,$G_{k}$\rbrack based on Priority Threshold PI
\EndProcedure
\end{algorithmic}
\end{algorithm}

\textbf{Step 1-3:} Same as the Algorithm 1

\textbf{Step 4:} For each gap $G_{i}$ in the spatial index, we perform a sliding window operation over the qualified gaps which satisfy the maximal union merge criteria by adding only new $GC_{i}$ $\in$ $G_{i}$ and then recomputing the AGM score for the entire union-merged group and also for the new gap. For instance, line 7 shows the overall operation after $G_{i}$ satisfies the maximal union merge criteria with the Union Group and appends new $GC_{i}$ $\in$ $G_{i}$. The sliding window operation ensures linear operations are given prior information about the left and right pointers along with $GC_{i}$. We then perform early termination based on Equation \ref{earlytermination}, where we check if the change in prior AGM score (i.e., union group $\not\in$ $G_{i}$) with new AGM score (i.e., $G_{i}$). If the change is not significant, we consider AGM of maximal union group $\in$ $G_{i}$; otherwise, we compute the AGM separately.

\textbf{Step 5:} Same as the Algorithm 1

\textbf{Execution Trace:} 
Given gaps <A>,<B>,<C>,<D>, and <E> in the same index, we first initialize the left (L) and right (R) pointers at the boundary of the MOBR along <A, B, C>, where <A, B, C> satisfies Intersection Bounds (IB). If gap <D> intersects with <A, B, C> and $\lambda \geq \theta$, we increment our right pointer (R) if the early termination criteria are satisfied. For instance, Figure \ref{fig:Exec_Trace_DRM} (b) in the middle shows that the difference between AGM(<A, B, C>) and AGM(<D>) is smaller than a pre-defined threshold $\delta$, i.e., similar to Equation \ref{earlytermination}. However, the rightmost box in Figure \ref{fig:Exec_Trace_DRM} (c) shows that Equation \ref{earlytermination} was not satisfied for <D> due to a change in the signal coverage map, which changes the abnormal gap measure for <A, B, C> and <D>. Hence, AGM(<A, B, C>) and AGM(<D>) are computed separately, and both left and right pointers are incremented by one position. 

\begin{figure}[ht]
    \centering
    \includegraphics[width=0.95\textwidth]{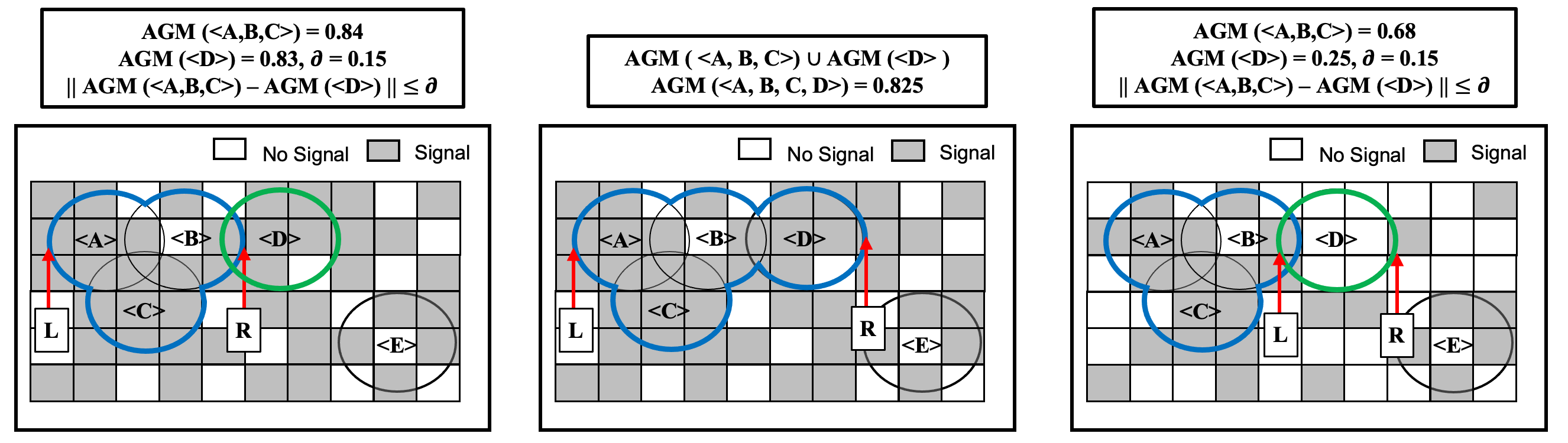}
    \caption{Dynamic Region Merge Execution Trace\\}
    \label{fig:Exec_Trace_DRM}
\end{figure}


\section{Theoretical Evaluation}
\label{section:TheoreticalEvaluation}
We evaluate our proposed algorithm STAGD for correctness and completeness in Section \ref{subsection:correct_complete}. Section \ref{subsection:TimeComplexity} further analyzed STAGD's time complexity compared with the baseline Memo-AGD for both average and worst case.
\subsection{Correctness and Completeness}
\label{subsection:correct_complete}
\begin{lemma}~\label{lemma:Memo-AGM_correct} 
The space-time-aware gap detection and dynamic region merge algorithms are correct.
\end{lemma}
\begin{proof}
Given a finite set of $|N|$ trajectory gaps, space-time aware gap detection (STAGD) loops around a finite number of times and terminates. The algorithm prunes on each trajectory gap and then conducts a spatial-temporal intersection producing a finite number of candidate pairs. 

In addition, the intersection boundary IB within candidate pairs does not allow pairs whose IB does not intersect with any of the candidate pairs. For instance, if $G_{i}^{IB}$ $\in$ $Polygon(G_{1}) \cap Polygon(G{2})$ and $G_{j}^{IB}$ $\in$ $Polygon(G_{3})$, then $G_{j}^{IB} \cap G_{i}^{IB}$ $\gets Polygon(G_{1}) \cap Polygon(G_{2}) \cap Polygon(G_{3}) \neq \emptyset$. This proves that $Polygon(G_{1}) \cap Polygon(G_{2})$, $Polygon(G_{2}) \cap Polygon(G_{3})$ and $Polygon(G_{1}) \cap Polygon(G_{3})$ $\neq \emptyset$ since the overall intersection of $Polygon(G_{1}) \cap Polygon(G_{2}) \cap Polygon(G_{3}) \neq \emptyset$. 

In STAGD, the Union Boundary (UB) considers candidate pairs which involved in intersection with other gaps within the entire subgroup (i.e., $Polygon(G_{1}) \cap Polygon(G_{2})$ or $Polygon(G_{2}) \cap Polygon(G_{3})$ or $Polygon(G_{1}) \cap Polygon(G_{3})$ $\in$ $\emptyset$). This may result in long chains which can be discontinued or broken by any of the proposed interest measures, such as the DO threshold and early termination criteria defined within the paper. 

In terms of AGM computations correctness, the proposed Dynamic Region Merge (DRM) Algorithm provides a correct output within a given signal coverage map by considering signal coverage variability at a given defined geographic space. The proposed interest measures (e.g., DO threshold) and early termination criteria ensure correct AGM computations over derived spatial bounds from both STAGD and DRM algorithms. Hence STAGD and DRM are correct.
\end{proof}

\begin{lemma}~\label{lemma:Memo-AGM_complete} 
Space time-aware gap detection and dynamic region merge are complete.
\end{lemma}
\begin{proof}
Given a finite set of $|N|$ gaps, STAGD and DRM covers all the gaps within their candidate pair sets. For instance, they ensure that the total number of gaps $G_{1}$, $G_{2}$, ..., $G_{n}$ which are participating in subsets is equal to the number of gaps in the Non-Observed List. For instance, Space Time-Aware Gap Detection algorithms execution trace output (i.e., <A, B, C, D> and <E>) the same number of gaps as defined in their Non-Observed List (i.e., <A>,<B>,<C>,<D>,<E>). This proves that both Space Time-Aware Gap Detection is complete. In addition, all gaps $G_{1}$, $G_{2}$, ... , $G_{n}$ belonging either to a maximal sets (e.g., <A,B,C>) or an independent set (e.g., <A>,<B>, and <C>) are considered. Hence, STAGD and DRM are complete.
\end{proof}
\subsection{Asymptotic Time Complexity}
\label{subsection:TimeComplexity}
We briefly discuss the asymptotic complexity of baseline Memo-AGD and the proposed STAGD+DRM algorithm based on preprocessing gap enumeration and AGM computations. However, both algorithms heavily depend upon gap distribution within a study area which indirectly depends upon the worst and best case. For instance, if independent gaps (i.e., gaps that are not participating in any maximal sets) have skewed distribution at a given index, then both algorithms has the same performance asymptotically. However, if gaps are merged and/or in uniform distribution, the proposed algorithms out-performs the baseline via indexing and DRM optimization. Hence, we will consider both the worst (a skewed distribution with independent gaps) and average-case (merged and independent gaps) for both algorithms.   

\textbf{Memo-AGD:} Given $n$ gaps, sorting operations require $\mathcal{O}(n\log{}n)$ time based on the $t_{start}$ of the gap. The Memo-AGD use caching based on $IB$ resulting in k maximal sub-groups in quadratic time-complexity. In the case of AGM computations, given k merged gaps and M X N grid cells, results in $k \times \mathcal{O}(M \times N)$ where (k << N). However, in the worst case, $\mathcal{O}(n^{2})$ and $n \times \mathcal{O}(M \times N)$ is performed for gap enumeration and AGM computations, respectively. 

\textbf{STAGD + DRM:} Given $n$ gaps, the space time-aware gap detection algorithm uses comparison-less sorting based on $t_{start}$ in linear time (i.e., $\mathcal{O}(n)$) followed by \textit{gap enumeration} and \textit{AGM computation}. In the case of gap enumerations, we index each gap in $\mathcal{O}(n\log{}n)$ time along with $k$ linear comparisons at the leaf-node level, where $k$ are merged groups ($k$ << $n$). This results in overall $\mathcal{O}(n\log{}n)$ + $\mathcal{O}(n \times k)$. In the case of AGM computations, the cell in polygon operation will speed up due to sliding window and memoization in the average case. However, if all gaps are independent and has skewed distribution in a single index (i.e., worst case), results in $\mathcal{O}(n)^{2}$ in comparison operations. In addition, given $N$ gaps and $M$ cells, we perform $N \times M$ operations in the worst case (where no gaps are intersecting). If gaps are involved in union operations with $k$ maximal groups, we perform $k \times M \times N$ operations where $k$ << $N$. Hence, the worst and average time complexity of DRM is $n \times \mathcal{O}(M \times N)$ and $k \times \mathcal{O}(M \times N)$, respectively.

\section{Experimental Validation}
\label{section:validation}
In this section, we present the comprehensive design of the experiments, including the related work, baseline, and proposed algorithms, which were outlined in Section \ref{subsection:ExperimentDesign}. We also discuss the real-world dataset \cite{aisdataUS} and the synthetic data generation method used for evaluating the solution quality of the algorithms. Furthermore, in Section \ref{subsection:Results}, we present the results of the experiments, which were based on accuracy and computation time, and varied depending on different parameters, such as the number of gaps and GPS points.

\subsection{Experiment Design}
\label{subsection:ExperimentDesign}

\textbf{Comparative Study:} The goal of the experiments was to assess and compare the solution quality of three methods: Linear Interpolation (Related Work), K Nearest Neighbor Imputation (Related Work), and STAGD+DRM (Proposed). In the case of execution time, we compared Memo-AGD (Baseline) and STAGD+DRM (Proposed). Both solution quality and computational efficiency were compared based on five parameters: the number of GPS points, the number of trajectory gaps, the effective missing period (EMP) (i.e., the minimum time period the object is missing), the average speed ($S_{avg}$), and the degree of overlap (DO) threshold. The detailed design of the experiments is presented in Figure \ref{fig:ExperimentDesign}.

\begin{figure*}[ht]
    \centering
    \includegraphics[width=0.65\textwidth]{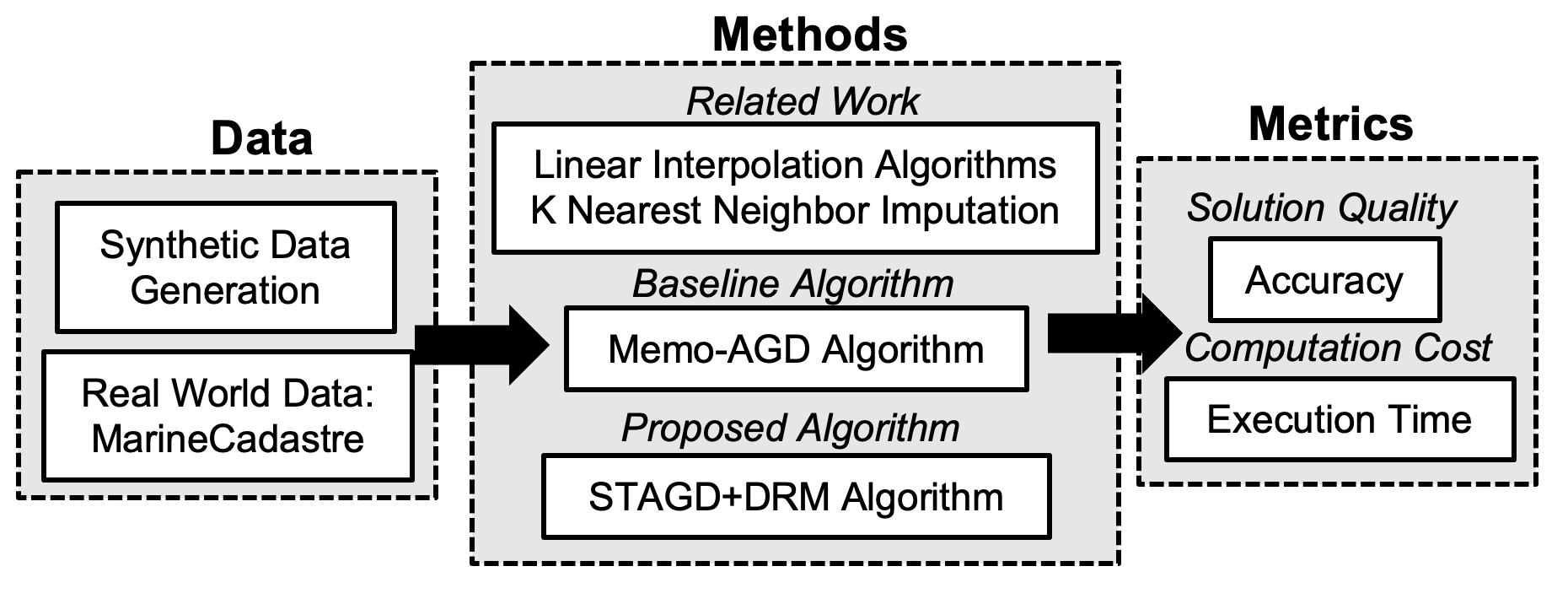}
    \captionsetup{justification=centering}
    \caption{Experiment Design}
    \label{fig:ExperimentDesign}
\end{figure*}

\textbf{KNN Imputation:} In this paper, we apply the k-Nearest Neighbors (k-NN) imputation method, previously utilized for anomaly detection in \cite{mazzarella2017novel}, to identify and address the issue of missing Automatic Identification System (AIS) messages, referred to as dropouts, for vessels at specific times. This method aims to reconstruct the missing data, including the estimated position of the vessel and the Received Signal Strength Indicator (RSSI). The vessel's position is initially estimated or predicted using tracking systems and models such as the Constant Velocity Model (CVM). Subsequently, we attempt to estimate the missing RSSI values. However, as RSSI data is not provided in \cite{mazzarella2017novel}, we employ a KNN imputer that calculates based on the inverse-weighted Euclidean distance from the predicted position to the fixed K neighbors. In our experimental setup, we chose K=5 and computed the weighted average from all existing neighbors to update the latitude-longitude coordinates derived using the Constant Velocity Model (CVM).

\textbf{Real World Dataset:} We utilized the real-world MarineCadastre \cite{aisdataUS} dataset to measure computational efficiency. The dataset contains various attributes, such as longitude, latitude, speed over ground, and others, for 150,000 objects that were recorded between 2009 and 2017. The geographical area covered by the data extends from 180W to 66W degrees longitude and from the 90S to 90N degrees latitude, based on the WGS 1984 coordinate system. In our experiments, we only considered the attributes of longitude, latitude, time, speed over ground (SOG), and Maritime Mobile Service Identity (MMSI). The maximum speed was calculated by taking the average SOG of the effective missing period (EMP). Figure \ref{fig:ER Diagram} shows an Entity Relationship (E-R) Diagram for three entities: vessels, trips, and location signals. A vessel can take multiple trips, and each trip can emit multiple location signals.
\begin{figure*}[ht]
    \centering
    \includegraphics[width=0.75\textwidth]{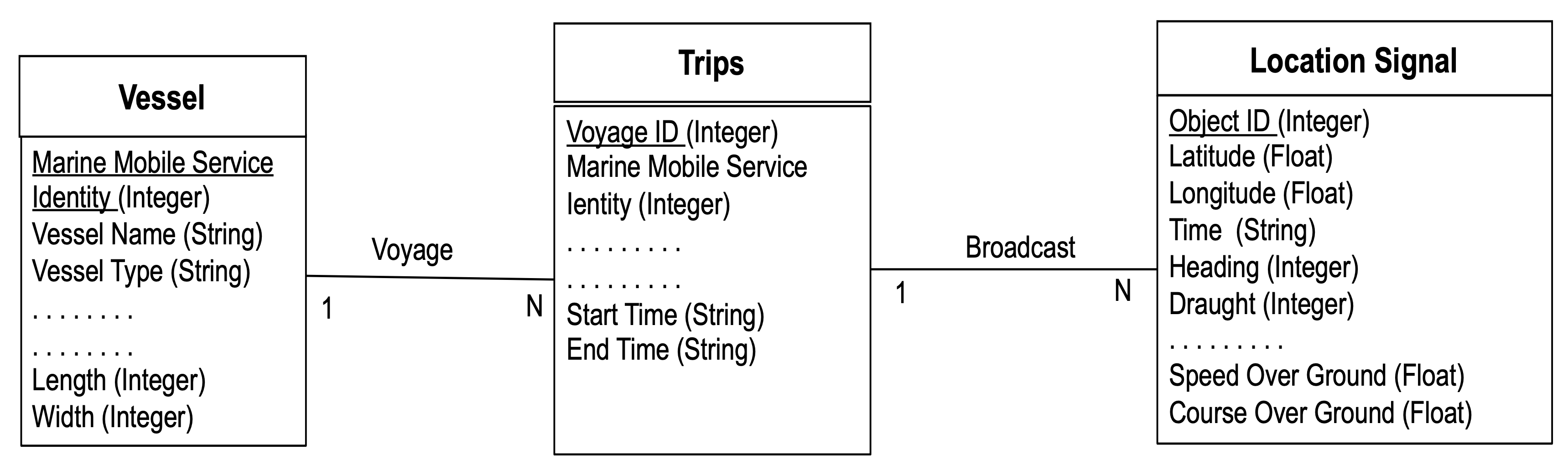}
    \captionsetup{justification=centering}
    \caption{Entity Relation diagram of the MarineCadastre \cite{aisdataUS} dataset}
    \label{fig:ER Diagram}
\end{figure*}


\textbf{Synthetic Dataset Generation:} The experiment for solution quality, lacked ground truth data, meaning we had no information on whether a gap in a trajectory was normal or abnormal. To overcome this, the proposed algorithms were evaluated using synthetic data created from a real-world dataset. The dataset consists of 1.25 x $10^{5}$ trajectories collected in the Bering Sea, covering a study area from $179.9W$ to $171W$ degrees longitude and from $50N$ to $58N$ degrees latitude, with 1500 trajectory gaps between 2014 and 2016. Figure \ref{fig:SyntheticDataGeneration} shows the detailed methodology for synthetic data generation.
 
\begin{figure}[ht]
    \centering
    \includegraphics[width=0.90\textwidth]{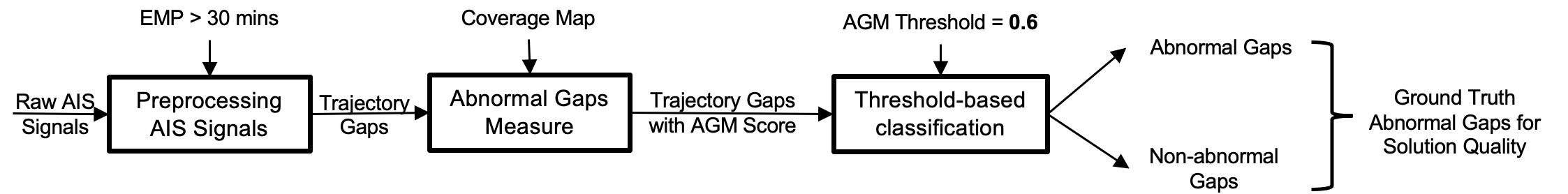}
    \caption{Synthetic Data Generation\\}
    \label{fig:SyntheticDataGeneration}
\end{figure}

The data was preprocessed to identify gaps lasting more than 30 minutes, which were considered trajectory gaps. The AGM (Abnormal Gap Measure) score was calculated using linear interpolation along with the proposed methods, and each gap was classified as either abnormal or non-abnormal based on an AGM threshold of 0.6. Gaps with an AGM score above 0.6 were deemed abnormal. The choice of the AGM score threshold was designed to reduce bias in the accuracy of the results, which was achieved by conducting multiple trial runs of both the baseline and proposed approaches.

\textbf{Computing Resources: } We performed our experiments on a system with a 2.6 GHz 6-Core Intel Core i7 processor and 16 GB 2667 MHz DDR4 RAM.

\subsection{Experiment Results:}
\label{subsection:Results}

\textbf{Solution Quality:} The proposed STAGD+DRM approach was compared with a Linear interpolation and KNN Imputation method that is widely used in the literature \cite{chen2013iboat} using space-time interpolation approach (also used in Memo-AGD \cite{sharma2022abnormal}) in terms of accuracy. The accuracy metric is defined in Equation \ref{ineq19} as the ratio of the actual trajectory gaps involved in abnormal behavior to the total number of abnormal gaps.
\begin{equation}
    \label{ineq19}
    Accuracy = \frac{Actual\ Gaps\ involved\ in\ Abnormal\ Behavior}{Total\ Gaps\ involved\ in\ Abnormal\ Behaviour}
\end{equation}
In Equation \ref{ineq19}, the "Actual Gaps" in the numerator are the "true positives" (correct abnormal gaps) and "true negatives" (correct non-abnormal gaps) generated by the two baseline approaches from the related work (Linear Interpolation and KNN Imputation) and the proposed approach (STAGD+DRM). Total Gaps in the denominator are all the gaps that are correctly and incorrectly classified as abnormal in a given study area.

The detailed sensitivity analysis based on different parameters for the related work, baseline, and the proposed algorithms are as follows:\\

\textbf{(1) GPS Points:} First, we fixed the EMP at 60 mins, the number of objects (or vessels) at 500, the speed at 20 $meters/second$, DO threshold to 0.5, and varied the number of GPS points from 2.5 × $10^{5}$ to 1.25 x $10^{6}$. The results in Figure \ref{subfig:Acc_GPS} show that linear interpolation under-performs against STAGD with Dynamic Region Merge(STAGD+DRM). Since DRM ensures to capture of more accurate gaps, especially in case of the high variability of signal coverage map. Linear Interpolation and KNN imputation captures abnormal gaps with less accuracy as compared to the proposed STAGD+DRM as GPS density grows.

\textbf{(2) Objects (Vessels):} We fixed the number of GPS points to $5.0 \times 10^{5}$, EMP to 60 mins, Speed at 20 $meters/second$, DO threshold to 0.5 and varied the number of objects from 2500 to 12500. Figure \ref{subfig:Acc_Obj} shows that the proposed approach captures more accurate signal coverage in terms of maximal subgroups as compared to linear threshold since it does not consider the DO Threshold. Linear interpolation and KNN imputation is less accurate since it does not consider space-time interpolation, which later incorrectly classifies abnormal gaps.

\textbf{(3) Effective Missing Period (EMP):} In this experiment, we set the number of objects (or vessels) to 900, the number of GPS points to $5.0 \times 10^{5}$, the Degree Overlap (DO) threshold to 0.5, the speed to 20 $meters/second$, and varied the effective missing period (EMP) from 2 to 10 hours. Figure \ref{subfig:Acc_EMP} shows that STAGD+DRM is again the more accurate since DRM provides a more precise capture of variable signal coverage in cases of large maximal union groups. With an increasing EMP, the larger gaps capture more variability in the spatial distribution of the reported cells, resulting in more accurate AGM scores. Linear Interpolation and KNN imputation both underperform in this measure.

\textbf{(4) Speed (meter/second):} The number of objects was set to 900, GPS points to $5.0 \times 10^{5}$, DO threshold to 0.5, speed to 20 $meters/second$, and the $S_{max}$ threshold was varied from $10$ to $50$ meters/second. Figure \ref{subfig:Acc_Speed} shows similar trends as EMP increases, the size of geo-ellipses also increases which later results in more potential interactions with other geo-ellipses. In addition, large geo-ellipses also result in more spatial distributions in terms of signal coverage which is better captured by STAGD+DRM and preserving its correctness due to the early stopping criterion property. Linear interpolation and KNN imputation shows no benefit because it does not take speed into account.

\textbf{(5) Degree Overlap (DO) Threshold ($\lambda$):} Last, the degree of overlap (DO) threshold was varied from $0.2$ to $1.0$ while keeping the number of objects at 900, the number of GPS points at $5.0 \times 10^{4}$, the speed at $20$ meters/second, and the EMP threshold at 60 mins. As seen in Figure \ref{subfig:Acc_DO}, STAGD+DRM algorithms display linear trends, with STAGD+DRM capturing signal coverage more accurately than linear interpolation. Linear interpolation and KNN imputation by contrast, shows no benefit because they both do not take the DO threshold into account.\\

\begin{figure}[ht]
    \centering 
\begin{subfigure}{0.30\textwidth}
  \includegraphics[width=\linewidth]{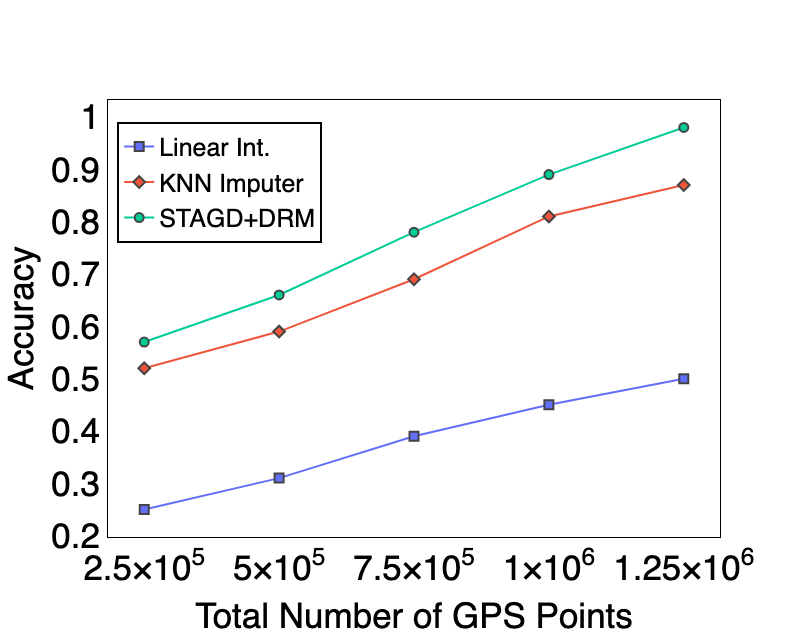}
  \captionsetup{justification=centering}
  \caption{GPS Points \break}
  \label{subfig:Acc_GPS}
\end{subfigure}\hfil 
\begin{subfigure}{0.30\textwidth}
  \includegraphics[width=\linewidth]{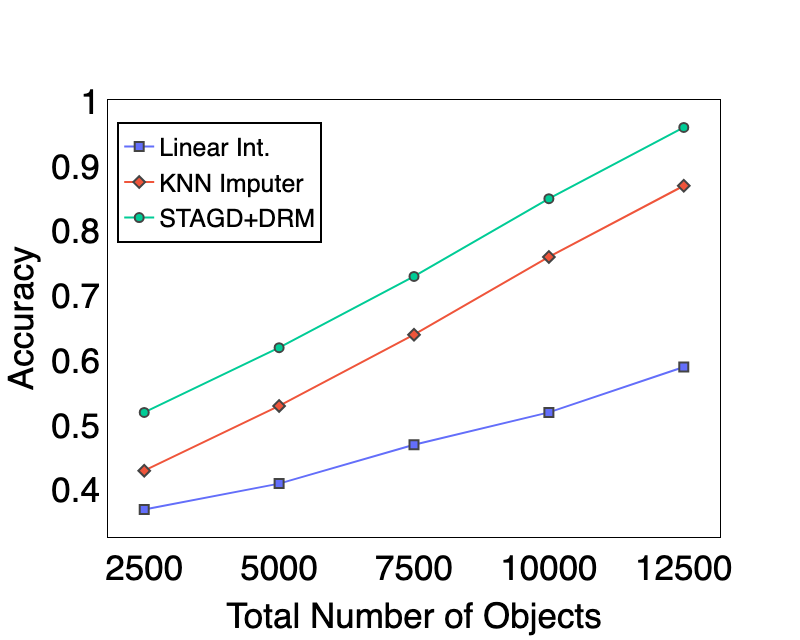}
  \captionsetup{justification=centering}
  \caption{Number of Objects \break}
  \label{subfig:Acc_Obj}
\end{subfigure}\hfil 
\begin{subfigure}{0.30\textwidth}
  \includegraphics[width=\linewidth]{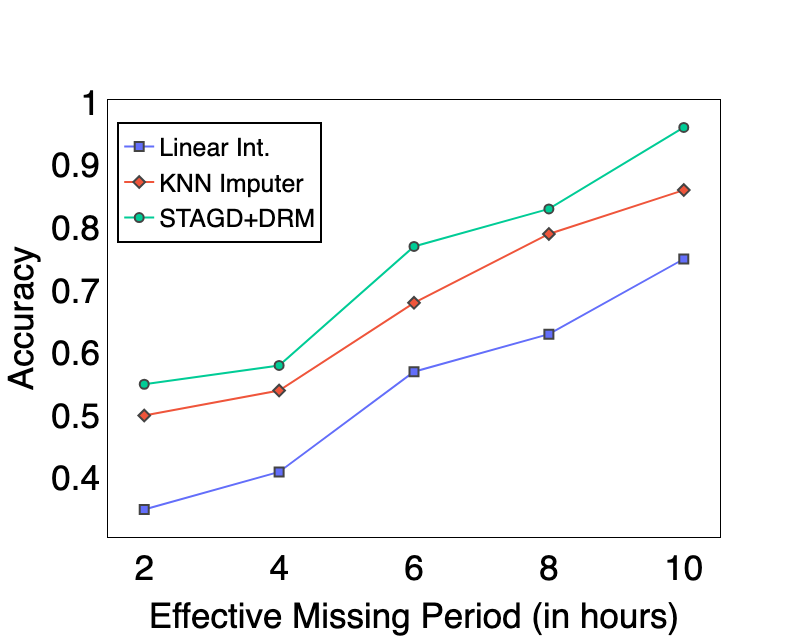}
  \captionsetup{justification=centering}
  \caption{Effective Missing Period \break}
  \label{subfig:Acc_EMP}
\end{subfigure}\hfil 
\begin{subfigure}{0.30\textwidth}
  \includegraphics[width=\linewidth]{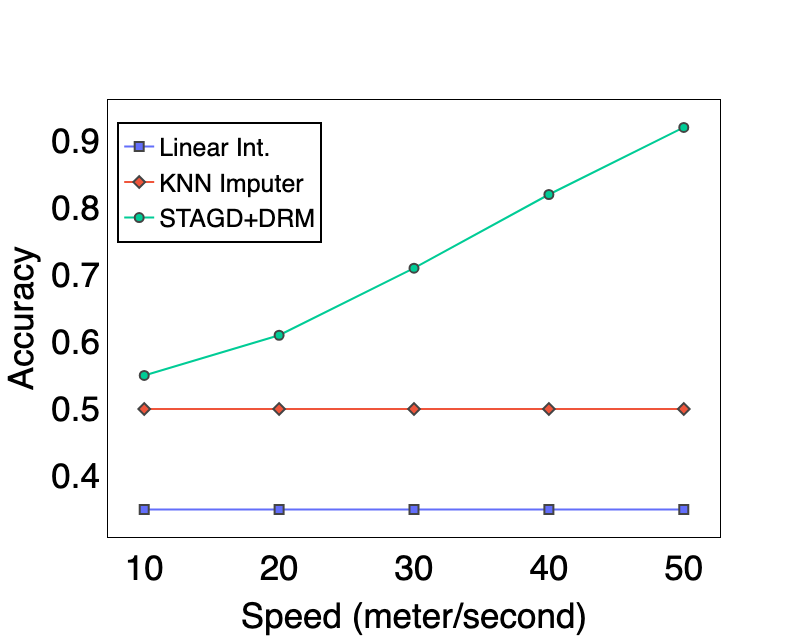}
  \captionsetup{justification=centering}
  \caption{Speed}
  \label{subfig:Acc_Speed}
\end{subfigure}\hfil 
\begin{subfigure}{0.30\textwidth}
  \includegraphics[width=\linewidth]{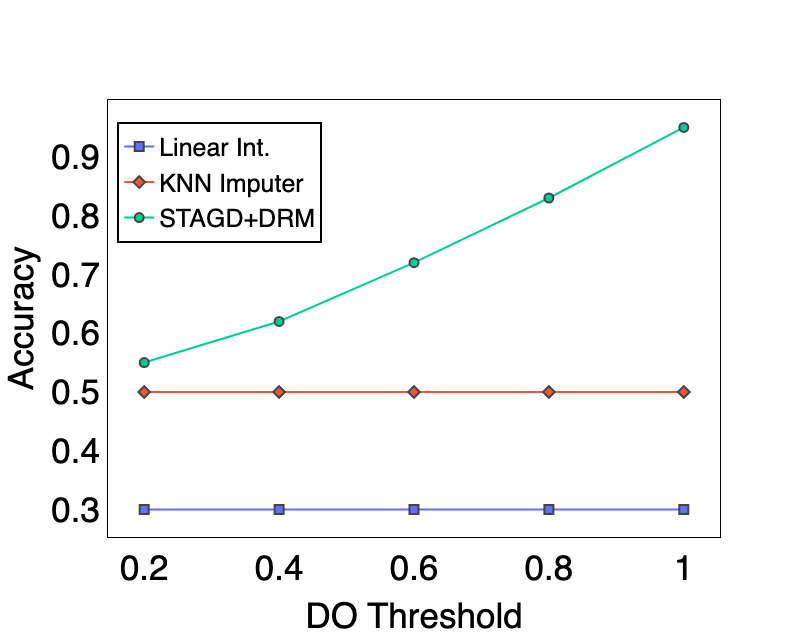}
  \captionsetup{justification=centering}
  \caption{DO Threshold}
  \label{subfig:Acc_DO}
\end{subfigure}\hfil 
\caption{Proposed STAGD+DRM is always more accurate than Linear Interpolation and KNN-Imputation under varying parameters.}
\captionsetup{belowskip=0pt}
\label{fig:SolutionQuality}
\end{figure}

\textbf{Computational Efficiency:} The proposed STAGD+DRM algorithm was compared to the baseline Memo-AGD in terms of computational efficiency using the MarineCadastre dataset \cite{aisdataUS}. Five different parameters were tested: number of GPS points, number of objects, effective missing period (EMP), maximum speed $S_{max}$ (the maximum speed achieved within the object's missing time period), and degree overlap (DO) threshold. The runtime was monitored while considering the indexing and AGM computation cost for STAGD+DRM.

\textbf{(1) Number of GPS Points:} First, we set the effective missing period (EMP) to 60 minutes, the maximum speed $S_{max}$ to $15$ meters/second, the degree overlap (DO) threshold to $0.2$, and varied the number of trajectory gaps from $500$ to $2500$. As shown in Figure \ref{subfig:Time_GPS}, STAGD+DRM consistently outperforms Memo-AGD due to comparison-less temporal indexing and reduced spatial comparisons in 3D R-Trees. Additionally, as the density of GPS points increases in the study area, the potential number of interactions between two or more geo-ellipses also increases.

\textbf{(2) Number of Objects:} In this experiment, we set the effective missing period (EMP) to 60 minutes, the maximum speed $S_{max}$ to $15$ meter/second, the number of GPS points to $5.0 \times 10^{5}$, and the DO threshold to $0.2$, and we varied the number of objects from $2500$ to $12500$. As shown in Figure \ref{subfig:Time_Obj}, the results are similar to when the number of GPS points was varied. STAGD+DRM outperforms Memo-AGD as it effectively indexes and computes AGM scores of gaps derived from different objects or vessels. 

\textbf{(3) Effective Missing Period (EMP):} Next, we set the number of objects to $2500$, the maximum speed $S_{max}$ to $15$ meters/second, the DO threshold to $0.2$, and the number of GPS points to $5.0 \times 10^{5}$. We then varied the EMP threshold from $2$ hours to $10$ hours. As shown in Figure \ref{subfig:Time_EMP}, STAGD+DRM runs the fastest. Increasing the EMP results in larger geo-ellipses and, thus, more potential interactions and larger maximal groups. 3D R-Trees, allow large gaps to be indexed, leading to a skewed distribution of gaps within the index. Both Memo-AGD and STAGD+DRM use the DO threshold to avoid these issues. However, STAGD's Dynamic Region Merge also allows for union and intersection operations and memoizes grid cells, resulting in improved performance compared to Memo-AGD.

\textbf{(4) Speed:} We set the number of objects to $2500$ and the effective missing period (EMP) to $2$ hours the DO threshold was set to $0.2$, and varied the maximum speed, $S_{max}$, from $10$ meters/second to $50$ meters/second. Again, as seen in Figure \ref{subfig:Time_Speed}, the results are similar to when the EMP was varied, as a high-speed value leads to larger geo-ellipses, resulting in more potential interactions. Furthermore, the proposed algorithm with STAGD+DRM's flexibility towards union and intersection operations outperforms the Memo-AGD algorithm.

\textbf{(5) DO Threshold ($\lambda$):} Last, we set the number of objects to 2500, the effective missing period to 2 hours, the maximum speed to 15 $m/s$, the number of GPS points to $5.0 \times 10^{5}$, and varied the DO threshold from $0.2$ to $1.0$. The proposed approach still outperformed Memo-AGD (Figure \ref{subfig:Time_DO}). A higher DO threshold meant that gaps were less likely to overlap, which led to fewer geo-ellipse clusters within the signal coverage area. While both Memo-AGD and STAGD+DRM allow for the merging of trajectory gaps, STAGD+DRM's early termination criteria also allow it to ignore all the isolated gaps within the same spatial index, thus saving time.

\begin{figure}[h!]
    \centering 
\begin{subfigure}{0.30\textwidth}
  \includegraphics[width=\linewidth]{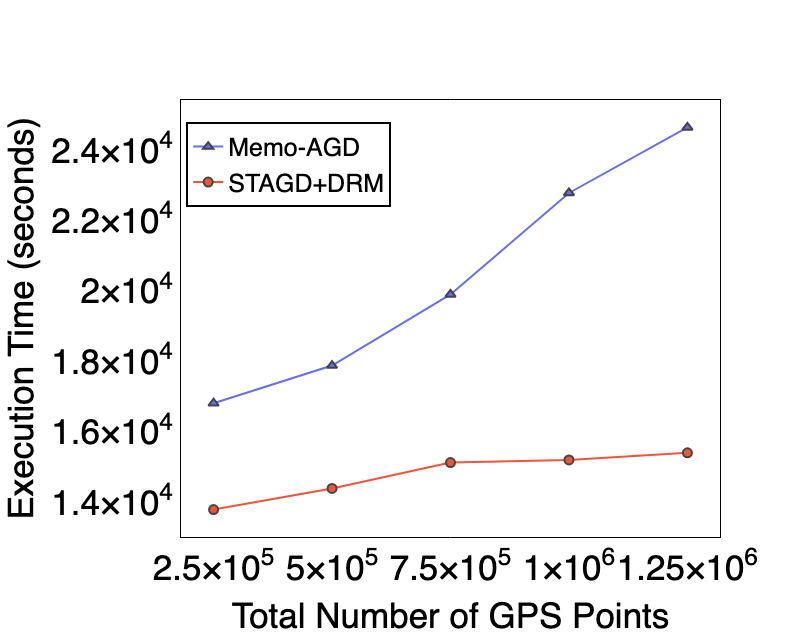}
  \captionsetup{justification=centering}
  \caption{GPS Points \break}
  \label{subfig:Time_GPS}
\end{subfigure}\hfil 
\begin{subfigure}{0.30\textwidth}
  \includegraphics[width=\linewidth]{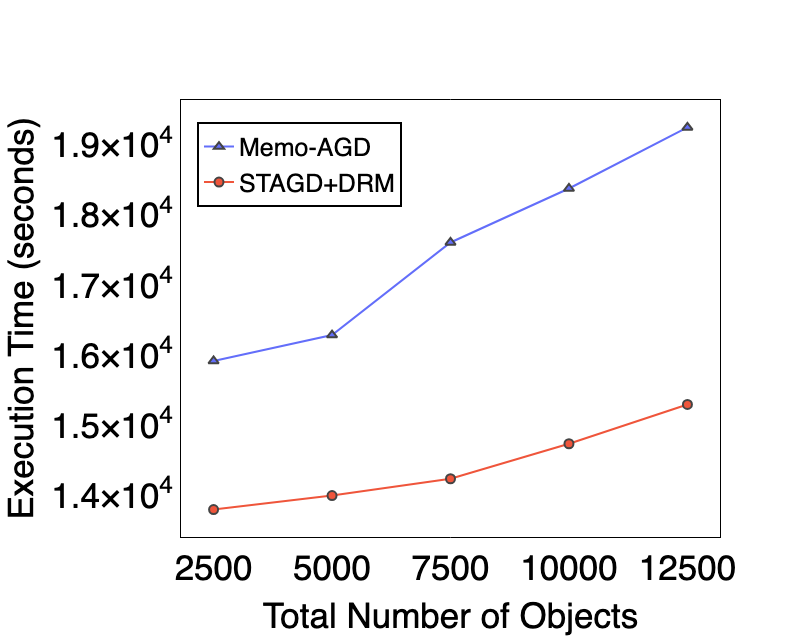}
  \captionsetup{justification=centering}
  \caption{Number of Objects \break}
  \label{subfig:Time_Obj}
\end{subfigure}\hfil 
\medskip
\begin{subfigure}{0.30\textwidth}
  \includegraphics[width=\linewidth]{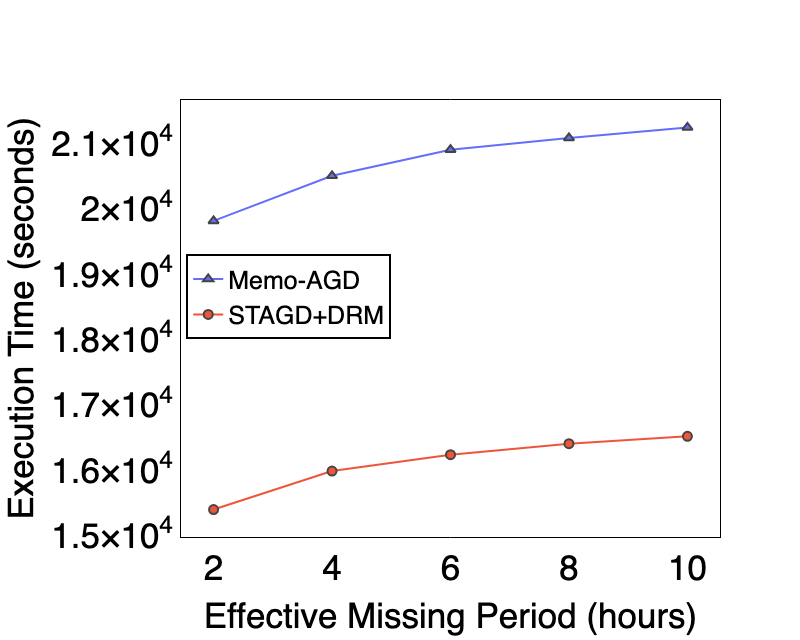}
  \captionsetup{justification=centering}
  \caption{Effective Missing Period \break}
  \label{subfig:Time_EMP}
\end{subfigure}\hfil 
\begin{subfigure}{0.30\textwidth}
  \includegraphics[width=\linewidth]{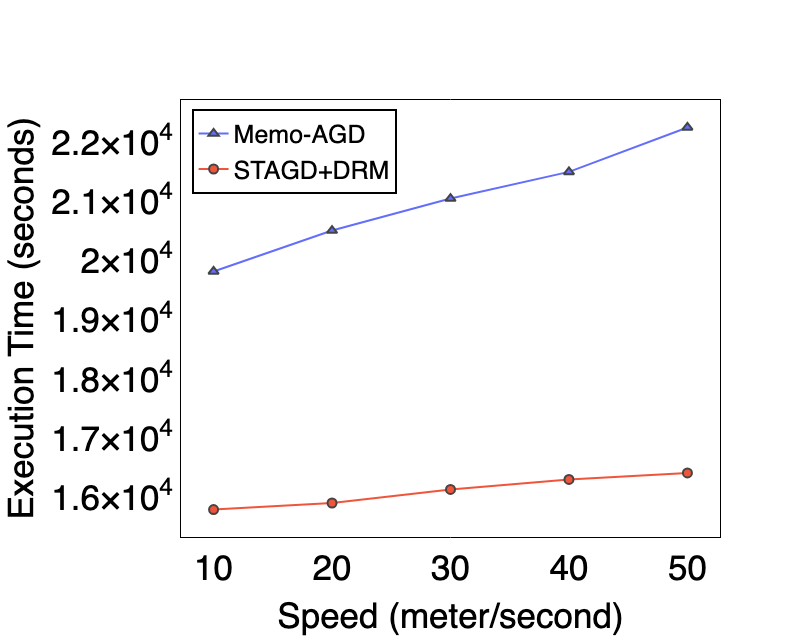}
  \captionsetup{justification=centering}
  \caption{Speed}
  \label{subfig:Time_Speed}
\end{subfigure}\hfil 
\begin{subfigure}{0.30\textwidth}
  \includegraphics[width=\linewidth]{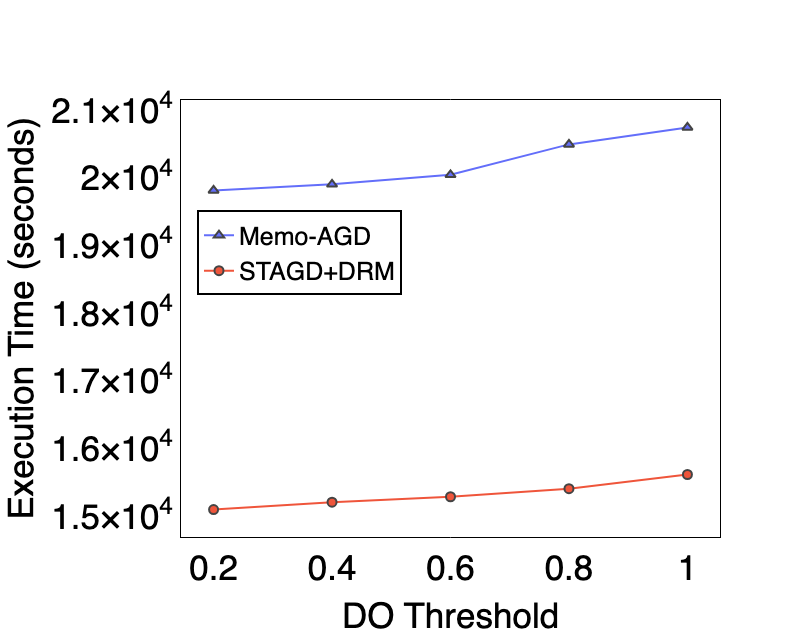}
  \captionsetup{justification=centering}
  \caption{DO Threshold}
  \label{subfig:Time_DO}
\end{subfigure}\hfil 
\caption{STAGD+DRM runs faster than Memo-AGD under different parameters}
\captionsetup{belowskip=0pt}
\label{fig:computationalefficiency}
\end{figure}

\subsection{Case Study}
We conducted a case study based on an actual event \cite{gibbens2018illegal} where a fishing vessel switched off its transponder while entering a protected marine habitat and then switched it back on after 15 days. We applied linear (shortest path) interpolation, KNN imputation, and our proposed space-time interpolation based algorithms on a MarineTraffic, ~\cite{aisdataUSNew} dataset to find abnormal regions in a study area ranging from 48.3 W to 65.8 W degrees in longitude and from 23.4 N to 29.6 N degrees in latitude near the Galapagos Marine Reserve. Figure \ref{subfig:CaseStudyNoInterpolation} shows that a ship switched off its transponder at point P in the Marine Reserve and switched it back on at point Q. Figure \ref{subfig:CaseStudyLinearInterpolation} assumes a linear interpolation where the vessel interacts with a small number of islands situated in the north resulting in an AGM score of 0.0 due to the absence of historical trajectories. Nearest Neighbor imputation shows that the vessels take a more frequented path than linear interpolation. Figure \ref{subfig:CaseStudySTInterpolation} shows the geo-ellipse formed by our space-time interpolation method. The elliptical region encompasses more historical trajectories, generating a higher AGM score (0.41) and thus meriting further investigation by human analysts.

\begin{figure}[h!]
    \centering 
\begin{subfigure}{0.30\textwidth}
  \includegraphics[width=\linewidth]{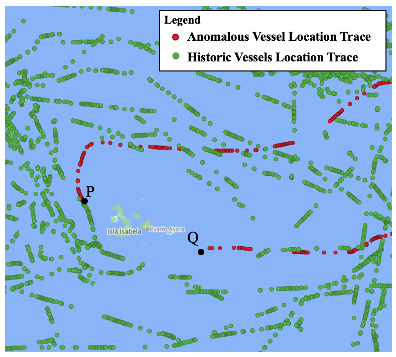}
  \captionsetup{justification=centering}
  \caption{No Interpolation \break}
  \label{subfig:CaseStudyNoInterpolation}
\end{subfigure}\hfil 
\begin{subfigure}{0.30\textwidth}
  \includegraphics[width=\linewidth]{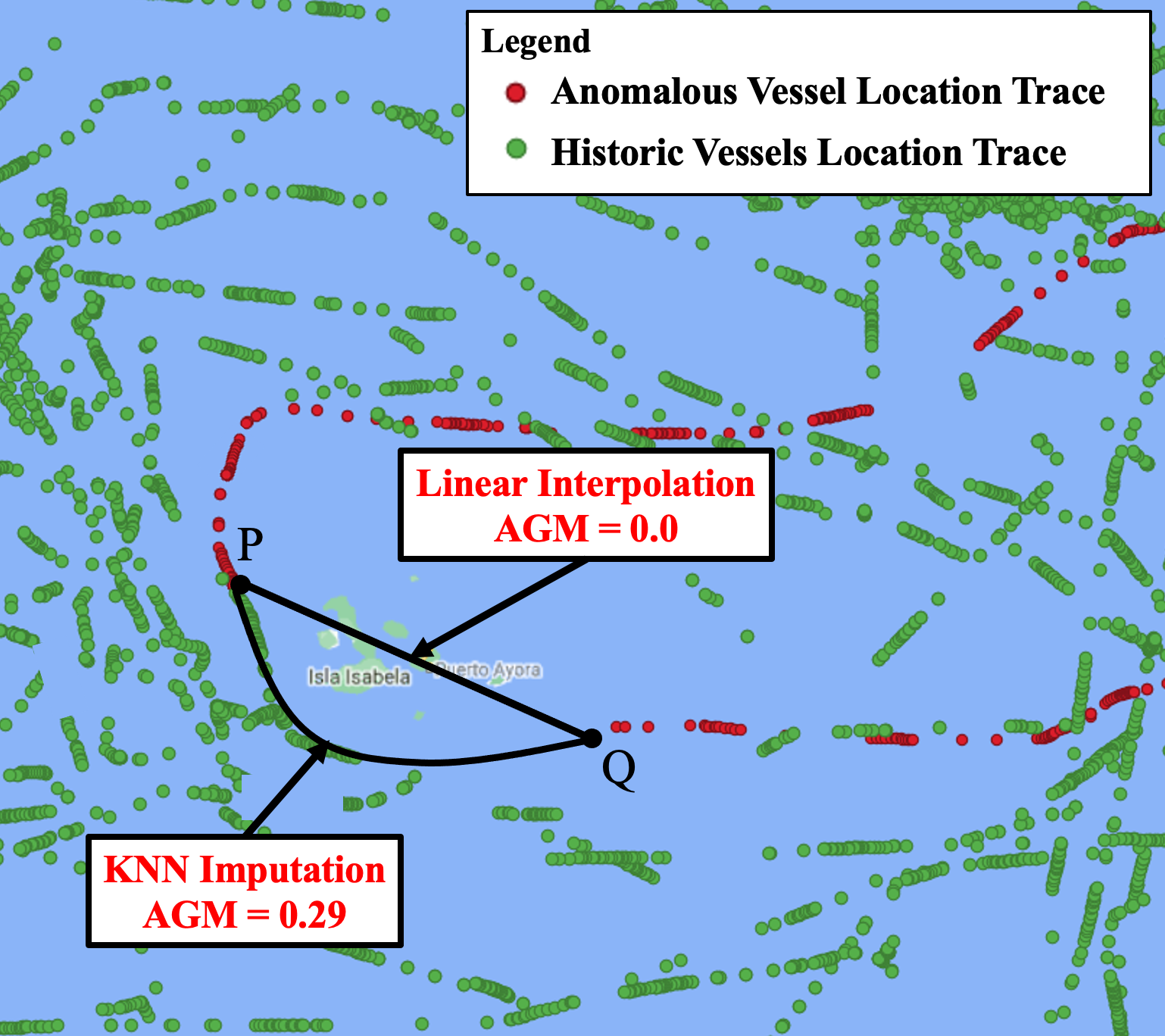}
  \captionsetup{justification=centering}
  \caption{Linear and k-NN Imputation \break}
  \label{subfig:CaseStudyLinearInterpolation}
\end{subfigure}\hfil 
\medskip
\begin{subfigure}{0.30\textwidth}
  \includegraphics[width=\linewidth]{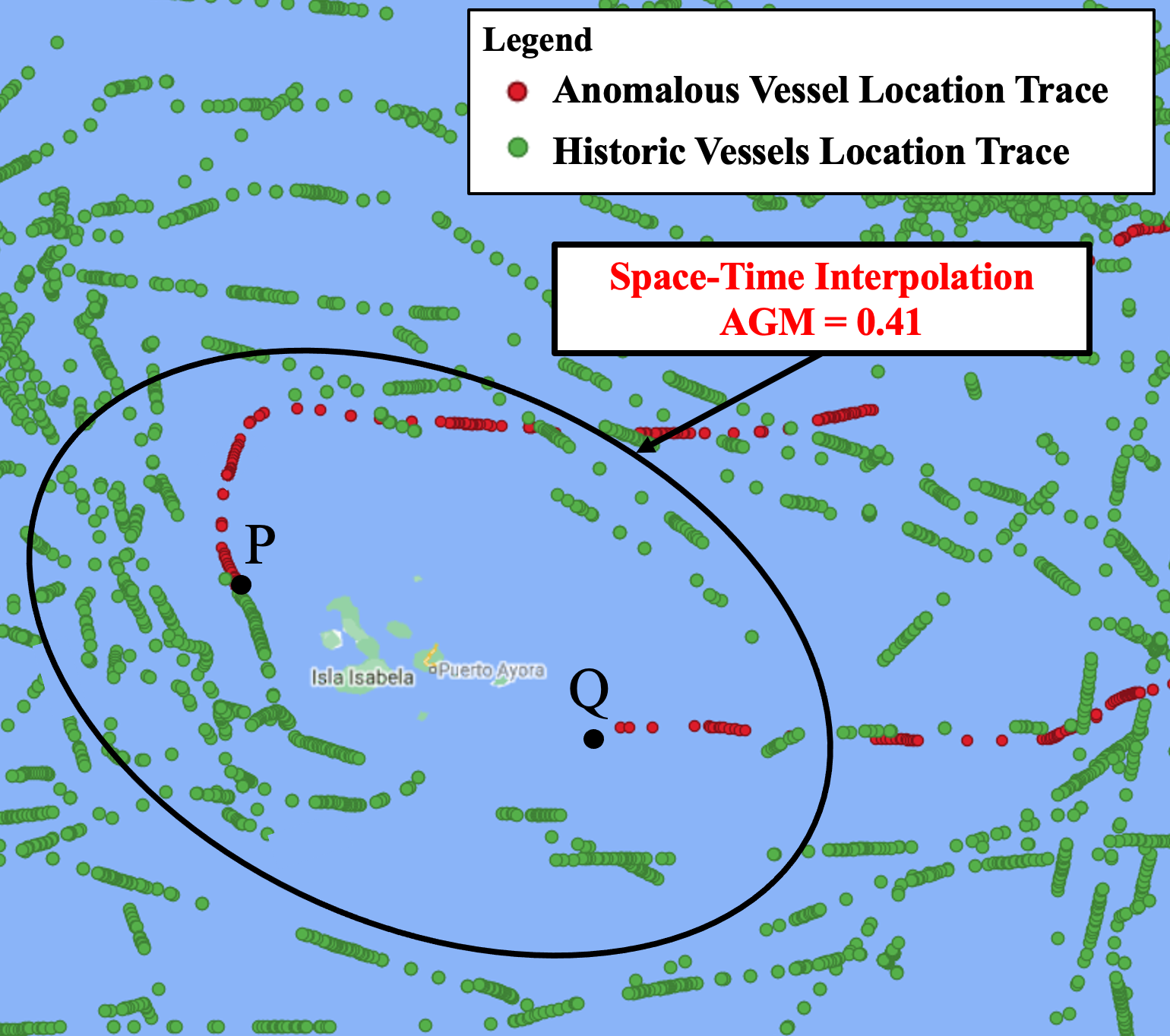}
  \captionsetup{justification=centering}
  \caption{Space Time Interpolation \break}
  \label{subfig:CaseStudySTInterpolation}
\end{subfigure}\hfil 
\caption{Comparison of k-NN Imputation, Linear Interpolation and Space-Time Interpolation}
\captionsetup{belowskip=0pt}
\label{fig:casestudy}
\end{figure}

\section{Discussion}
\label{section:Discussion}
\textbf{Signal Coverage Map:} Each grid cell $GC_{i}$ is classified as either reported or not reported by the signal coverage map based on whether its total location trace count exceeds a threshold, $\theta_e$. However, this threshold value $\theta_e$ is influenced by factors such as geographic region and the time frame for historical location traces. For example, accurately estimating signal coverage becomes difficult when a vessel enters a region with limited or no historical trace data, leading to a cold start problem. A signal coverage map is heavily dependent on the time frame range used to calculate it. For instance, a lower threshold value is needed if the time frame for global signal coverage is short, while a higher threshold value is required if the time frame is long. In this study, we only considered historical location traces based on the longest voyage time frame within a specific geographic area. Additionally, differences in geographic regions can affect the computation of the signal coverage map. For example, polar regions may have low signal coverage compared to more dense regions such as ports. Even within denser regions, signal coverage may vary due to factors such as weather or the designation of an area as a protected habitat.\\

\textbf{Abnormal Gap Measure (AGM):} AGM scores heavily depend on the geometric constraints of the geo-ellipse estimations within trajectory gaps. Similar to a classical space-time prism, a geo-ellipse is limited to the maximum speed achieved by the object within the duration of the trajectory gap. We did consider other measures for scoring anomalies that had to be abandoned because the data and ground truth information we needed was not available. However, they are worth considering:

\begin{itemize}
\item \textbf{Attributes:} Abnormality can also be defined with a combination of voyage-specific and physical attributes (e.g., speed, headings, draft). For instance, a draft is inversely proportional to speed since loading of cargo lower the ship's speed but raises the draught and vice-versa. Anomaly detection based on directional change (i.e., sudden change in movement direction) can also be extended toward trajectory gaps. However, post-processing human interpretation is required along with ground truth verification. 

\item \textbf{Temporal reasoning:} Including time of day as a factor can refine the proposed abnormality measure. For instance, an increased number of gaps during night-time may be an alarm signal to authorities. The addition of this factor would further enrich the score and open the way towards more sophisticated measures of abnormality based on emerging and vanishing hotspots of trajectory gaps.

\item \textbf{Speed Distribution:} Monitoring a distribution based on an object's velocity profile may result in multiple geo-ellipses, which more accurately quantify the continuous space within the geo-ellipse derived from the object's maximum speed. A physical interpretation based on dwell time and speed could be demonstrated by studying a vessel's temporal profile based on previous routes or voyages. For instance, since dwell time and speed are inversely related, a ship within a gap could spend more time maneuvering in the shortest path and move at a lower speed compared to a ship reaching the boundary region of a geo-ellipse at max speed.\\
\end{itemize}

\textbf{Probabilistic Interpretation:} Monitoring the distribution of location traces also helps define probabilistic interpretations of trajectory gaps. For instance, Gaussian processes may provide the assumption of availability of location trace based on the Gaussian distribution instead of the precise location. Other probabilistic interpretations within a geo-ellipse can further be considered as an average case in contrast to geo-ellipse (worst case) and shortest path (best case) interpretations. \\

\textbf{Kernel Density Estimation:} Kernel density estimation (KDE) is a non-parametric way to estimate the probability density function of a random variable. KDE transforms point data into a continuous surface, estimating the density of points across the area. It is particularly useful in various fields like ecology, urban planning, and criminology to identify clusters or "hotspots" of activity or occurrences. Incorporating time as an additional dimension in kernel density estimation (KDE) enables the analysis of spatiotemporal data, revealing dynamic patterns, trends, and hotspots that change over both space and time. This approach involves preparing spatiotemporal data with precise timestamps, selecting appropriate spatial and temporal kernel functions, and optimizing bandwidths for both dimensions to accurately capture the scale of patterns.


\section{Related Work}
\label{section:Related_work}
The extensive amount of trajectory data being generated via location-acquisition services today has spurred new research in anomaly detection, pattern recognition, etc. The survey in \cite{zheng2015trajectory} gives a broad overview of trajectory data mining techniques related to pre-processing, data management, uncertainty, pattern mining, etc. Methods for detecting anomalous trajectories are included in a broad classification of anomaly detection methods \cite{chandola2009anomaly}, which also includes trajectory anomaly detection. The literature \cite{dodge2008towards} also provides a comprehensive taxonomy of movement patterns, classifying them as primitive, compound, individual, group, etc. There are several methods for determining the similarity of trajectories, including edit distance \cite{chen2005robust}, longest common sub-sequence (LCSS) \cite{lin1995fast}, and dynamic time warping (DTW) \cite{yi1998efficient}. A recent technique called TS-Join \cite{shang2017trajectory} uses both spatial and temporal factors in a two-step process to determine similarity. Another method, Strain-Join \cite{ta2017signature}, uses a signature-based approach to identify similar trajectories by creating unique signatures and eliminating dissimilar ones. Other techniques include using a sliding window approach \cite{bakalov2005efficient,chen2009design} to identify potential candidates and filtering out non-similar trajectories using time interval thresholds \cite{bakalov2006continuous,ding2008efficient}. For trajectories with silent periods, Clue-aware Trajectory Similarity \cite{hung2015clustering} uses clustering and aggregation based on movement behavior patterns. In terms of processing queries for spatiotemporal joins, methods based on 3D R-Trees, such as Spatiotemporal R-Tree and TB-Tree \cite{pfoser2000novel}, is used. However, none of the literature has addressed anomaly detection within trajectory gaps. In the maritime domain, Riveiro et al. \cite{riveiro2018maritime} provide a generalized view of maritime anomaly detection but do not cover trajectory gaps.

Uncertainty within trajectory gaps can be quantified based on behavioral patterns of moving objects, where many techniques perform linear interpolation assuming shortest path discovery. There are many realistic frameworks \cite{pallotta2013vessel,chen2013iboat,lei2016framework} that do quantify uncertainty within trajectory gaps, but they are loosely based on shortest path discovery. For instance, one study \cite{pallotta2013vessel} derives knowledge of maritime traffic in an unsupervised way to detect low-likelihood behaviors and predict a vessel's future positions. However, the movement behavior is assumed linear and relies on historical mobility behavior, which requires extensive training data. 

Besides the shortest path, the paper \cite{mazzarella2017novel} uses k-Nearest Neighbors (k-NN) imputation to estimate missing Received Signal Strength Indicator (RSSI) values in AIS vessel trajectory data. When a dropout (missing AIS message) is detected, the vessel's position is first predicted using a tracking system like the Constant Velocity Model (CVM). Then, the k-NN method finds the k AIS messages nearest in position to the predicted vessel position, and the RSSI value of the single nearest AIS message (k=1) is used as the imputed RSSI for the missing message, leveraging the spatial correlation in the trajectory data. A variant of KNN \cite{chen2001jackknife} was also studied, which is a resampling method used to estimate the variance of a statistic, such as a sample mean. It is useful when the theoretical variance is difficult to derive or when the data has been subject to complex procedures like imputation. Bayesian Networks \cite{hruschka2007bayesian} were used, where multiple Bayesian networks are constructed, one for each attribute with missing values. However, all of the above methods either consider the shortest path or the most frequented path, which does not take into account other object movement capabilities.

The space-time prism model \cite{miller1991modelling,pfoser1999capturing} is a more realistic estimation of trajectory gap because it identifies a larger spatial region surrounding a gap where an object could have potentially deviated from the predefined linear path. A number of applications based on joins \cite{pfoser1999capturing}, rendezvous patterns \cite{sharma2020analyzing,sharma2022towards,sharma2022analyzing}, and probabilistic interpretation \cite{winter2010directed} have been proposed, but they are limited to theoretical simulations and need real-world validation. Other probabilistic frameworks \cite{cheng2008probabilistic,trajcevski2010uncertain} incorporate space-time prisms but do not address the abnormal mobility behavior of an object. A recent work \cite{zhao2021applying} considers trajectory gaps but only focuses on the space-time prism and does not consider abnormal behavior within gaps. This paper uses a space-time prism model to capture abnormal gaps based on historical data, which allows us to distinguish abnormal gaps for possible anomaly hypotheses.

\section{Conclusion and Future Work}
\label{section:Conclusion_future_work}
Given multiple trajectory gaps and a signal coverage map based on historical object activity, we find
possible abnormal gaps, which are later sent to human analysts for ground truth verification (inspired by Human-in-the-loop framework in \cite{sharma2022understanding}). The problem is important since it addresses societal applications such as improving maritime safety and regulatory enforcement (e.g., detecting illegal fishing, trans-shipments, etc.) and many other application domains. Current methods rely on linear interpolation and our baseline algorithm Memo-AGD \cite{sharma2022abnormal} using geo-ellipses estimation to capture abnormal gaps. 

In this paper, we extended our previous work by proposing a space-time aware gap detection (STAGD) algorithm which leverages space-time partitioning to index and filters out non-intersecting gaps effectively. In addition, we incorporate a dynamic region merge (DRM) to compute abnormal gap measures more efficiently while preserving correctness and completeness. The results show better computation time and accuracy for the proposed STAGD+DRM compared to the baseline linear interpolation and Memo-AGD.


\textbf{Future Work:} Computing abnormal gaps in trajectories is time-consuming and challenging. We are investigating new indexing methods that are more efficient than hierarchical indexing and linear search and better suited for modeling gaps in regional space. We intend to use these techniques in other areas of application. Our plan also include other methods (e.g., Baeysian Estimation \cite{hruschka2007bayesian}, regression-based \cite{fox2015applied}) and provide more accurate estimation of signal coverage density using kernel density estimation and identifying spatiotemporal hotspots \cite{sharma2022spatiotemporal} with time as an additional dimension. In addition, we aim to validate our evaluation metrics with real-world data and consider other external factors (e.g., device malfunction, hardware failure, etc.). We also aim to improve our measurement of unusual gaps by incorporating more physics-based parameters (e.g., draft) and exploring more sophisticated models (e.g., kinetic prisms). Finally, we plan to study ways to identify anomalies that are based on deception, such as fake trajectories or gaps generated by users who manipulate their location devices.

\begin{acks}
This research is funded by an academic grant from the National Geospatial-Intelligence Agency (Award No. HM0476-20-1-0009, Project Title: Identifying Aberration Patterns in Multi-attribute Trajectory Data with Gaps). Approved for public release, N-CERTS NGA-U-2023-01847. We would also like to thank Kim Koffolt and the spatial computing research group at the University of Minnesota for their helpful comments and refinements.
\end{acks}

\bibliographystyle{ACM-Reference-Format}
\bibliography{sample-base}

\appendix

\section{Execution Trace of Memo-AGD Algorithm}
\label{Exec-Trace:Memo-AGD}
Figure \ref{fig:ProposedExection} shows the execution trace of the Memo-AGD algorithm. At $t_{s}$ = 1, Step 1 initializes $G_{i}$ with variables $G_{i}^{LU}$ and $G_{i}^{Obs}$ which are later added to the empty Observed List. At $t_{s}$ = 2, <$A^{Obs}$> copies the \textit{current} Observed List (i.e., only <A>) to <$B^{Obs}$> after satisfying spatial and temporal overlap conditions. If true, we perform Step 2 where the resultant shape of <$A$> $\cup$ <$B$> is saved in both <A> and <B>. The variables LU gets updated with <A,B> for both <$A^{LU}$> and <$B^{LU}$>. Finally, removing <A> from <$B^{Obs}$> results in an empty list, and the loop terminates. At $t_{s}$ = 3, <C> interacts with the resultant shape of both <A> and <B> (i.e., <$A$> $\cup$ <$B$> and result <$A^{LU}$>,<$B^{LU}$> and <$C^{LU}$> as <A,B,C> (i.e., maximal sets). However, at $t_{s}$ = 8, <D> does not intersect spatially with <A> and will skip intersection operations with other elements of $A^{LU}$ (i.e., <B> and <C>) via $G^{Obs}$ -$G^{LU}$ operation and gets added to the Observed List and result in performance speedup as compared to the AGD algorithm. A similar operation is done at  $t_{s}$ = 9, where <E> does not interact with <A> but does interact with <D>, resulting in $D^{LU}$ = <D, E>. Finally, <F> does not interact with any of the elements, causing it to be directly saved as a maximal set <F>. The rest of the steps remain the same as the baseline, except we gather all maximal sets from the lookup table of each $G^{LU}$.

\begin{figure*}[htp]
    \includegraphics[width=1.00\textwidth]{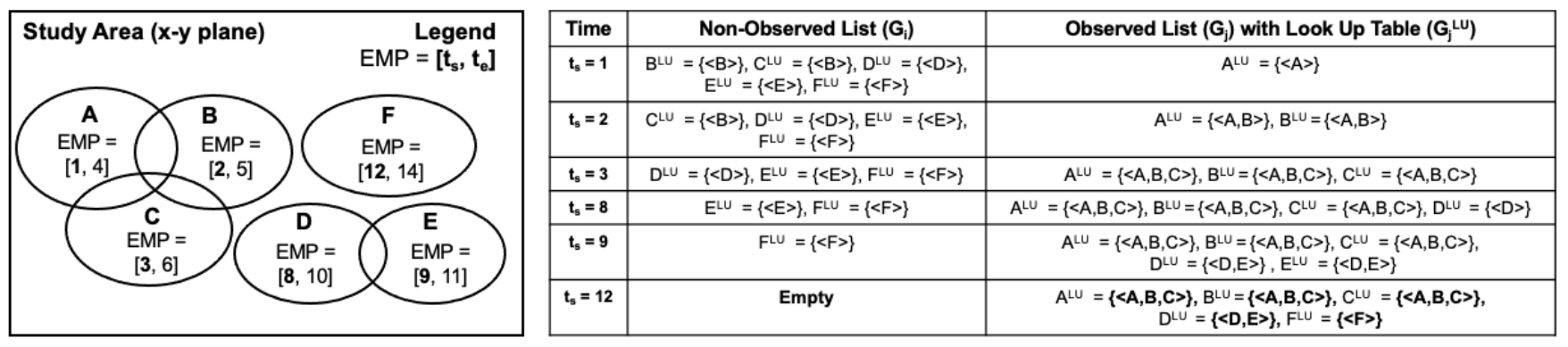}
    \centering
    \captionsetup{justification=centering}
    \caption{Execution Trace of Proposed Memo-AGD Algorithm}
    \label{fig:ProposedExection}
\end{figure*}

\section{Acceleration interpretation with Kinetic Prism Model}
\label{Kinetic Prism:acceleration}
\begin{figure}[h]
    \centering
    \includegraphics[width=0.75\textwidth]{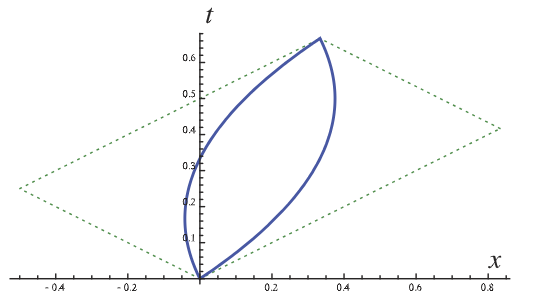}
    \caption{An example of an overlay of a kinetic prism (blue lines) with a classical prism (green dotted lines) \cite{kuijpers2009modeling}.\\}
    \label{fig:kinetic}
\end{figure}

Incorporating acceleration into the interpolation process results in a smoother and more quadratic curve within the interpolated region, compared to the traditional space-time prism approach. As depicted in Figure \ref{fig:kinetic}, we illustrate a conventional space-time prism with dashed green lines, contrasted against a kinetic prism outlined by solid blue lines. This comparison reveals a notable decrease in the area encompassed by the kinetic prism as opposed to the classical one, attributable to the minor temporal gap between the anchor points. The kinetic prism's approximation of a classical space-time prism improves with increased permissible acceleration or with an expansion in the time interval between anchor points when considered in relation to their spatial separation.

\end{document}